\documentclass[a4paper, 11pt]{article}

\usepackage{latexsym,mathrsfs}
\usepackage{amsmath,amssymb}
\usepackage{amsthm,enumerate,verbatim}
\usepackage{amsfonts}
\usepackage{graphicx}
\usepackage{algorithm}
\usepackage{algorithmic}
\usepackage{url}

\renewcommand{\algorithmicrequire}{\textbf{Input:}}
\renewcommand{\algorithmicensure}{\textbf{Output:}}

\newtheorem{lemma}{Lemma}

\newtheorem{prop}{Proposition} 
\newtheorem{theorem}{Theorem}
\newtheorem{corollary}{Corollary}

\newtheorem{remark}{Remark}

\DeclareMathOperator{\conv}{conv} 
\DeclareMathOperator{\argmax}{argmax} 
\DeclareMathOperator{\argmin}{argmin} 
\DeclareMathOperator{\cone}{cone}

\DeclareMathOperator{\diag}{diag}  
\DeclareMathOperator{\tr}{tr}

\title{Robustness Analysis of Hottopixx, \\ a Linear Programming Model for Factoring Nonnegative Matrices} 
\date{}

\author{Nicolas Gillis\thanks{E-mail: nicolas.gillis@uclouvain.be. The author is a postdoctoral researcher of the Fonds de la Recherche Scientifique (F.R.S.-FNRS). This paper presents research results of the Belgian Network DYSCO (Dynamical Systems, Control, and Optimization), funded by the Interuniversity Attraction Poles Programme initiated by the Belgian Science Policy Office.} \\ 
ICTEAM Institute, Universit\'e catholique de Louvain, \\ B-1348 Louvain-la-Neuve, Belgium.} 

\begin{document}

\maketitle

% **** To DO ****
% 
% - citations

\begin{abstract}

Although nonnegative matrix factorization (NMF) is NP-hard in general, it has been shown very recently that it is tractable under the assumption that the input nonnegative data matrix is close to being separable (separability requires that all columns of the input matrix belongs to the cone spanned by a small subset of these columns). Since then, several algorithms have been designed to handle this subclass of NMF problems. In particular, Bittorf, Recht, R\'e and Tropp (`Factoring nonnegative matrices with linear programs', NIPS 2012) proposed a linear programming model, referred to as Hottopixx. In this paper, we provide a new and more general robustness analysis of their method. In particular, we design a provably more robust variant using a post-processing strategy which allows us to deal with duplicates and near duplicates in the dataset.  

\end{abstract} 

\textbf{Keywords.} Nonnegative matrix factorization, separability, robustness to noise, linear programming, Hottopixx.

\section{Introduction}

Nonnegative matrix factorization (NMF) is a popular machine learning technique and allows one to express a set of nonnegative vectors as nonnegative linear combinations of nonnegative basis elements \cite{LS99}. More formally, given a nonnegative matrix $M \in \mathbb{R}^{m \times n}_+$ corresponding to $n$ vectors in an $m$-dimensional space and a factorization rank $r$, the aim is to find a basis matrix $U \in \mathbb{R}^{m \times r}_+$ and a weight matrix $V \in \mathbb{R}^{r \times n}_+$ such that the norm of the error $M-UV$ is minimized. 
Although NMF is NP-hard~\cite{V09}, Arora et al.~\cite{AGKM11} recently showed that it can be solved in polynomial time given that the matrix $M$ is close to being separable.  
A nonnegative matrix $M \in \mathbb{R}^{m \times n}_+$ is $r$-separable if and only if it can be expressed as $M = WH$, where $W \in \mathbb{R}^{m \times r}_+$, $H \in \mathbb{R}^{r \times n}_+$, and each column of $W$ is equal to a column of $M$. In other terms, $M \in \mathbb{R}^{m \times n}_+$ is $r$-separable if and only if 
\[
M \; = \; 
W \, [ I_r, H'] \Pi 
\; = \;
[W, WH'] \Pi,  
\]
for some $H' \in \mathbb{R}^{r \times n}_+$ and some permutation matrix $\Pi \in \{0,1\}^{n \times n}$. 
Any nonnegative matrix is $n$-separable because of the trivial decomposition $M = MI_n$ with $r = n$, and the aim is to find a decomposition where $r$ is as small as possible. 
It is rather straightforward to check that the smallest such $r$ %to obtain an exact factorization 
is the number of extreme rays of the cone generated by the columns of $M$, that is, $\cone(M) = \{ Mx \ | \ x \in \mathbb{R}^n_+ \}$. 
Equivalently, if the columns of matrix $M$ are normalized to sum to one, the smallest such $r$ is the number of vertices of the convex hull of the columns of $M$, that is, $\conv(M) = \{ Mx \ | \ x \in \mathbb{R}^n_+, \sum_{i=1}^n x_i = 1 \}$; see \cite{KSK12} and the references therein for more details about the geometric interpretation of the separable NMF problem. 

It turns out that the separability assumption makes sense in several practical situations. 
For example, in document classification, each column of $M$ corresponds to a document (that is, a vector of word counts) and is approximated with a nonnegative linear combination of the columns of matrix $W$ which correspond to different topics (that is, bags of words). %(the weights are given by the columns of the  matrix $H$). 
Separability of $M$ requires that, for each topic, there exists at least one document discussing only that topic. In practice, this condition is not often satisfied and it is more reasonable to assume separability of $M^T$ (that is, each row of $H$ is equal to a row of $M$) which requires that, \emph{for each topic, there exists at least one word used only by that topic}; see the discussions in \cite{AGKM11, AGM12}.  The separability assumption is also widely used in hyperspectral imaging and is referred to as the \emph{pure-pixel assumption}; see \cite{GV12} and the references therein.

In practice, the input separable matrix $M$ is perturbed with some noise and it is therefore desirable to design robust algorithms; see \cite{AGKM11, AGM12, BRRT12, ESV1294, EMO12, GV12, KSK12}. In fact, in the noiseless case, the problem is rather easy and reduces to identifying the vertices of the convex hull of a set of points. 
%and several have been proposed recently \cite{AGKM11, AGM12, EMO12, ESV1294, BRRT12, GV12, KSK12}. 
In this paper, we will focus on the algorithm of Bittorf, Recht, R\'e and Tropp~\cite{BRRT12}, referred to as Hottopixx, which is described in the next section.  
As we will see, the robustness analysis provided by the authors is rather restrictive as it does not deal with duplicates nor near duplicates of the columns of $W$ in the dataset: the aim of this paper is to develop a more general analysis of their algorithm, and design a provably more robust variant applicable to any noisy separable matrix (that is, allowing duplicates and near duplicates in the data).

\subsection{Hottopixx: a Linear Programming Model for Separable NMF}  \label{hottop}

From now on, we will \emph{always} assume that the columns of the input data matrix $M$ have been normalized in order to sum to one, that is, 
\begin{itemize} 
\item[(i)] The zero columns of $M$ have been discarded, and 
\item[(ii)] Each column of $M$ is updated using $M(:,j) \leftarrow \frac{M(:,j)}{||M(:,j)||_1}$.  
\end{itemize} 
We will also always assume that we are given a noisy separable matrix $\tilde{M} = M+N$
where $N \in \mathbb{R}^{m \times n}$ is some noise added to the separable matrix $M$ such that  
\[
||N||_{1} = \max_{||x||_1 \leq 1} ||Nx||_1 = \max_j ||N(:,j)||_1  \leq \epsilon, \quad \text{ for some } \epsilon \geq 0. 
\] 

The matrix $M$ is $r$-separable if and only if 
\begin{align}  
M = WH & = W [I_r, H'] \Pi =  [W, WH'] \Pi \nonumber  \\ 
&
= [W, WH'] \Pi \, \underbrace{ \Pi^{-1} \left( \begin{array}{cc} 
I_{r} & H'  \\ 
0_{(n-r) \times r} & 0_{(n-r) \times (n-r)} \end{array} \right) \Pi}_{X^0 \in \mathbb{R}^{n \times n}_+}  = MX^0, \label{sepM} 
\end{align} 
for some $W \geq 0$, $H' \geq 0$ and some permutation matrix $\Pi$. Equation~\eqref{sepM} shows that  $M$ is $r$-separable if and only if there exists a nonnegative matrix $X^0 \in \mathbb{R}^{n \times n}_+$ such that: (1)~$X^0$~contains $(n-r)$ all-zero rows and the $r$-by-$r$ identity matrix as a submatrix (up to permutation), and (2)~$M = MX^0$. Notice that because the columns of matrix $M$ and $W$ sum to one, the columns of the matrix $H'$ have sum to one as well. 
Based on these observations, Bittorf et al.~\cite{BRRT12} proposed to solve the following optimization problem\footnote{In \cite{BRRT12}, the model assumes separability of $M^T$ so that \eqref{rechtLP} is equivalent to the model in \cite{BRRT12} applied to $M^T$. We prefer here to work with the columns.} 
in order to identifying approximately the columns of the matrix $W$ among the columns of the matrix $\tilde{M}$: 
\begin{subequations}  \label{rechtLP}
\begin{align} 
\min_{X \in \mathbb{R}^{n \times n}_+} & \quad p^T \diag(X)  \nonumber \\ 
\text{ such that } & \quad ||\tilde{M}-\tilde{M}X||_{1} \leq 2 \epsilon, \label{rLPa}\\ 
%&  \quad e^T \diag(X) = r_1, e^T \diag(Y) = r_2 ,  \\ 
&  \quad \tr(X) = r,  \label{rLPb} \\  
&  \quad X(i,i) \leq 1  \text{ for all } i ,  \label{rLPc} \\
&  \quad X(i,j) \leq X(i,i) \text{ for all } i,j , \label{rLPd}
\end{align}
\end{subequations}  % Use Eq.\@? 
where $p \in \mathbb{R}^n$ is any vector with distinct entries. Intuitively, the model reads as follows: we have to assign a weight in [0,1] (Equation~\ref{rLPc}) to each column of $M$ (that is, give a value to $X(i,i)$ for all $i$) for a total weight of~$r$ (Equation~\ref{rLPb}). Moreover, we cannot use a column to reconstruct another column with a weight larger than the corresponding diagonal entry of $X$ (Equation~\ref{rLPd}), while we have to guarantee that the approximation error is small (Equation~\ref{rLPa}).  
It is interesting to notice that the problem is always feasible: in fact,  
$X^0$ from Equation~\eqref{sepM} is a feasible solution of \eqref{rechtLP} since the columns of $H'$ sum to one and 
\begin{align*}
||\tilde{M} -\tilde{M} X^0||_{1} 
& = ||(M+N) - (M+N) X^0||_{1} \\
&  \leq ||M - M X^0||_{1} + ||N||_{1} + ||N||_{1} ||X^0||_{1} \leq 2 \epsilon. 
\end{align*} 
Finally, Bittorf et al.~\cite{BRRT12} identify approximately the columns of $W$ by selecting the $r$ columns of $\tilde{M}$ whose corresponding diagonal entries of an optimal solution of \eqref{rechtLP} are the largest; see Algorithm~\ref{balgo2}, referred to as \emph{Hottopixx}. Note that the corresponding optimal weight matrix $H$ can be obtained by solving another linear program; see Algorithm~\ref{balgo}. 
\algsetup{indent=2em}
\begin{algorithm}[ht!]
\caption{Hottopixx - Extracting Columns of a Separable Matrix by Linear Programming \cite{BRRT12} \label{balgo2}}
\begin{algorithmic}[1] 
\REQUIRE A noisy $r$-separable matrix $\tilde{M} = WH + N$, the noise level $||N||_{1} \leq \epsilon$ and the number $r$ of columns of $W$. 
\ENSURE A matrix $\tilde{W}$ such that $||\tilde{W}(:,P)-{W}||_{1}$ is small for some permutation $P$. 
%A matrix $\tilde{W}$ such that $||\tilde{W}(:,P)-{W}||_{1} \leq  \delta$ for some permutation $P$ and some $\delta \geq 0$. \medskip 
\STATE Find the optimal solution $X^*$ of \eqref{rechtLP}. 
\STATE Let $\mathcal{K}$ be the index set corresponding to the $r$ largest diagonal entries of $X^*$. 
\STATE Set $\tilde{W} = \tilde{M}(:,\mathcal{K})$. 
\end{algorithmic} 
\end{algorithm} 
\algsetup{indent=2em}
\begin{algorithm}[ht!]
\caption{Near-Separable NMF using Hottopixx and Linear Programming \cite{BRRT12} \label{balgo}}
\begin{algorithmic}[1] 
\REQUIRE A noisy $r$-separable matrix $\tilde{M} = WH + N$, the noise level $||N||_{1} \leq \epsilon$ and the number $r$ of columns of $W$.  
\ENSURE An nonnegative factorization $(\tilde{W}, \tilde{H})$ such that $||\tilde{M}-\tilde{W}\tilde{H}||_{1}$ is small. \medskip 
\STATE Compute $\tilde{W}$ using Algorithm~\ref{balgo2}.  
\STATE Solve $\tilde{H} = \argmin_{Y \geq 0} ||\tilde{M}-\tilde{W}Y||_{1}$. 
\end{algorithmic} 
\end{algorithm} 

Before stating robustness results, it is important to define the conditioning of matrix $W$, which is a crucial characteristic of separable NMF problems. 
%Recall the aim is to locate the vertices of the convex hull of the columns of $M$.  
%Therefore, 
In fact, the better the columns of $W$ are spread in the unit simplex $\Delta^m = \{ x \in \mathbb{R}^m \ | \ x~\geq~0, \sum_{i=1}^m = 1 \}$, the more noise tolerant the data will be. 
In~\cite{AGKM11, BRRT12}, this conditioning is measured via the following parameter: 
\[
\alpha \; = \min_{1 \leq k \leq r, x \in \Delta^{r-1}} ||W(:,k) - W(:,\mathcal{R})x||_1, \quad \text{ where } \mathcal{R} = \{1,2,\dots,r\} \backslash \{k\}, 
\]
and the matrix $W$ is said to be \emph{$\alpha$-robustly simplicial}. (Notice that $\alpha \leq 2$ for any nonnegative matrix $W$ whose columns sum to one.) 
In other words, $\alpha$ is the minimum among the $\ell_1$-distances between a column of $W$ and the convex hull of the other columns of $W$. 
It is necessary that  $||N||_1 \leq \epsilon < \frac{\alpha}{2}$ for \emph{any} separable NMF algorithm to be able to approximately recover the columns of $W$ from the matrix $\tilde{M} = WH + N$. 
In fact, if $\epsilon \geq \frac{\alpha}{2}$, any $r$-separable matrix $M$ with $r \geq 2$ can be perturbed so that one of the columns of the perturbed matrix $\tilde{M}$ corresponding to a column of $W$  belongs to the convex hull of the other columns. In other words, we can perturb the matrix $M$ so that it becomes $(r-1)$-separable and we could therefore not  distinguish one of the columns of $W$ from the columns of $M$. For example, with 
\[
W = \left(
\begin{array}{c}
\frac{\alpha}{2} I_{r}  \\
(1-\frac{\alpha}{2})  e^T % \\ 
%0_{1\times r}  
\end{array} 
\right), 
H = I_r  
\text{ and } 
N = \left(
\begin{array}{c}
\frac{-\alpha}{2} I_{r}  \\
0 % \\ 
%0_{1\times r}  
\end{array} 
\right), \text{ we have } 
\tilde{M} =  \left(
\begin{array}{c}
 0_{r \times r}  \\
(1-\frac{\alpha}{2})  e^T % \\ 
%0_{1\times r}  
\end{array} 
\right),
\] 
so that the matrix $M = WH$ is $r$-separable with $W$ $\alpha$-robustly simplicial and $||N||_1 = \frac{\alpha}{2}$, while $\tilde{M}$ is $1$-separable.

In order to prove robustness results for Algorithm~\ref{balgo}, Bittorf et al.~\cite{BRRT12} used the following observation: 
\begin{lemma} \label{lem1}
Suppose $M$ is normalized and admits a rank-$r$ separable factorization $WH$, and suppose $\tilde{M} = M+N$ with $||N||_{1} \leq \epsilon$. 
If  %Algorithm~\ref{balgo2} is able to extract a matrix 
$\tilde{W}$ is such that $||\tilde{W}(:,P)-{W}||_{1} \leq  \delta$ for some $\delta \geq 0$ and some permutation $P$, then Algorithm~\ref{balgo} constructs a factorization $(\tilde{W}, \tilde{H})$ satisfying $||\tilde{M} - \tilde{W} \tilde{H}||_{1} \leq  \epsilon + \delta$. 
\end{lemma}
\begin{proof}
Denoting $N_W = \tilde{W}(:,P) - W$, we have 
\begin{align*}
||\tilde{M} - \tilde{W} \tilde{H}||_{1} 
& = \argmin_{Y \geq 0} ||M + N - (W+N_W) Y||_{1} \\
& \leq ||{M} + N - (W+N_W) H||_{1} \\
& \leq  ||N||_{1} + ||N_W H||_{1}  + ||{M} - W H||_{1} \\
& \leq \epsilon + \delta, 
\end{align*}
since the columns of $H$ sum to one. 
\end{proof}

Lemma~\ref{lem1} allows us to focusing on proving robustness results for Algorithm~\ref{balgo2}. In fact, any result that applies to Algorithm~\ref{balgo2} directly applies to Algorithm~\ref{balgo}. 
In this paper, we will therefore focus our attention on Algorithm~\ref{balgo2}, as it was implicitly done in~\cite{BRRT12}. \\ 

We can now state the robustness result for Hottopixx proposed in \cite{BRRT12}: 
\begin{theorem}[\cite{BRRT12}, Th.\@ 3.2] \label{prop3} 
Suppose $\tilde{M} = M+N$ where $M$ is normalized, and admits a rank-$r$ separable factorization $WH$ with $W$ $\alpha$-robustly simplicial 
and $||N||_{1} \leq \epsilon$. Suppose that the there is no duplicate of the columns of $W$, and that for all columns with index $j$ such that $M(:,j) \neq W(:,k)$ for all~$k$, we have a margin constraint $||M(:,j) - W(:,k)||_1 \geq d_0$ for all $k$. 
Suppose also that $\epsilon < \frac{ \min(\alpha d_0,\alpha^2)}{9(r+1)}$. 
Then Algorithm~\ref{balgo2} identifies correctly the columns of $W$, that is, it extracts a matrix $\tilde{W}$ satisfying $||\tilde{W}(:,P) - {W}||_{1} \leq  \epsilon$ for some permutation $P$.  \vspace{0.2cm} 
\end{theorem} 
  
Note that the above robustness result only deals with input data matrices \emph{without duplicates nor near duplicates} of the columns of $W$ (because of the margin constraint). In other terms, the columns of $W$ must be isolated for the robustness result to apply.  This is a very unnatural condition. For example, in document datasets, %each column of the input matrix corresponds to a word and each row to a document hence 
separability requires that, for each topic, there exists at least one word used only by that topic; see Introduction. 
In this context, the additional margin condition requires that, for each topic, there exists \emph{one and only one} word associated with that topic, which is rather impractical. In this paper, we propose a post-processing strategy for Hottopixx so that duplicates and near duplicates are allowed in the dataset.

\subsection{Conditioning and $\kappa$-Robustly Conical Matrices}

Because the columns of the variable $X$ in \eqref{rechtLP} are not required to sum to one, it turns out that it will be easier to work with the following parameter measuring the conditioning of matrix $W$: 
\[
\kappa = \min_{1 \leq k \leq r} \min_{x \in \mathbb{R}^{r-1}_+} ||W(:,k) - W(:,\mathcal{R})x||_1, \quad \text{ where } \mathcal{R} = \{1,2,\dots,r\} \backslash \{k\}, 
\]
and the matrix $W$ is said to be \emph{$\kappa$-robustly conical}. We have that $\kappa$ is the minimum among the $\ell_1$-distances between a column of $W$ and the convex \emph{cone} generated by the other columns of $W$. 
If the columns of $W$ sum to one (which will always be assumed), $\kappa \leq 1$ and we can relate $\alpha$ and $\kappa$ as follows: 
\begin{theorem} \label{kappaalpha}
For any $\alpha$-robustly simplicial and $\kappa$-robustly conical nonnegative matrix $W$ whose columns sum to one, we have 
\[
 \kappa \quad \leq \quad \alpha \quad \leq \quad 2 \kappa. 
\]
\end{theorem}
\begin{proof}
The first inequality follows directly from the definition. The second is proved in Appendix~\ref{appa}.  
\end{proof}

Therefore, it is essentially equivalent to working with $\alpha$ or $\kappa$ as they differ by a multiplicative factor of at most 2.

\subsection{Contribution and Outline of the Paper} \label{outline}

In this paper, we provide a new analysis of Hottopixx (Algorithm~\ref{balgo2}). This in turn allows us to design a post-processing strategy leading to a provably more robust variant (Algorithm~\ref{postpro}) which is applicable to any separable matrix (that is, duplicates and near duplicates are allowed in the dataset as opposed to the original robustness result from \cite{BRRT12}). \\  

In the first part of the paper (Section~\ref{ip33} ), we analyze the case where no duplicates nor near duplicates are allowed in the dataset, and focus on the following proposition: 
%, that is, that the maximum entry of matrix $H'$ in \eqref{sepM} is strictly smaller than one: 
\begin{prop} \label{prop1} %Suppose $M$ admits a rank-$r$ separable factorization $WH$ with $W$ $\alpha$-robustly simplicial 
Suppose $\tilde{M} = M+N$ where $M$ is normalized, admits a rank-$r$ separable factorization $WH$ where $W$ is $\kappa$-robustly simplicial with $\kappa > 0$, and has the form \eqref{sepM} with  
$\max_{i,j} H'_{ij} \leq \beta \leq 1$. 
Suppose also that $||N||_{1} \leq \epsilon$ and $\epsilon$ is sufficiently small. 
Then Algorithm~\ref{balgo2} extracts a matrix $\tilde{W}$ satisfying 
%\footnote{In the original paper \cite{BRRT12}, the robustness result only requires $||\tilde{W}(:,P) - W||_{1} \leq 2\epsilon$. However, their proof actually implies $||\tilde{W}(:,P) - W||_{1} \leq \epsilon$.} 
$||W - \tilde{W}(:,P)||_{1} \leq \epsilon$ 
for some permutation~$P$. \vspace{0.1cm} 
\end{prop} 

\noindent Note that the condition on the entries of $H'$ is implied by the margin constraint of Theorem~\ref{prop3} since 
\[
||M(:,j) - W(:,k)||_1 \geq d_0 \; \text{ for all } 1 \leq k \leq r 
\quad  \Rightarrow  \quad 
\max_i H(i,j) \leq \beta = 1 - \frac{d_0}{2} , 
\] 
see Lemma~\ref{lem4}. Hence, by Theorem~\ref{prop3}, Proposition~\ref{prop1} holds for $\epsilon < \frac{ \min(2 \alpha (1-\beta),\alpha^2)}{9(r+1)}$. 
In Section~\ref{ip33}, we prove that 
\begin{itemize}
\item $\epsilon \leq \frac{\kappa (1- \beta)}{9(r+1)}$ is sufficient for Proposition~\ref{prop1} to hold (Theorem~\ref{Th2}), while 
\item $\epsilon <  \frac{\kappa (1-\beta)}{(1-\beta)(r-1)+1}$ is necessary  for Proposition~\ref{prop1} to hold for any $r\geq 3$ and $\beta < 1$ (Theorem~\ref{nc1}). 
\end{itemize} 
Hence, our analysis gets rid of the term $\alpha^2$ from Theorem~\ref{prop3}, and is close to being tight.

In the second part of the paper (Section~\ref{ip2}), we do not make any assumption on the input separable matrix, and focus on the following proposition: 
\begin{prop} \label{prop2}
Suppose $\tilde{M} = M+N$ where $M$ is normalized and admits a rank-$r$ separable factorization $WH$ where $W$ $\kappa$-robustly simplicial with $\kappa > 0$. 
Suppose also that $||N||_{1} \leq \epsilon$ and $\epsilon$ is sufficiently small.  
Then Algorithm~\ref{balgo2} extracts a matrix $\tilde{W}$ satisfying $||W - \tilde{W}(:,P)||_{1} \leq \delta$ for some permutation $P$ and some $\delta \geq 0$. \vspace{0.1cm} 
\end{prop} 

\noindent We first show that it is necessary for Proposition~\ref{prop2} to hold  that $\epsilon <  \frac{\kappa}{r-1}$ (Corollary~\ref{cor2}), and that $\delta \geq 3 \frac{\epsilon}{\alpha} + \frac{3}{2} \epsilon$ for any $\epsilon < \frac{\alpha}{2}$ (Theorem~\ref{lbdelta}). (We also show that this lower bound on $\delta$ applies to a broader class of separable NMF algorithms.) 
Then, we propose a post-processing of the solution of the linear program \eqref{rechtLP} (see Algorithm~\ref{postpro}) for which the following result holds: \vspace{0.1cm} 

\noindent \textbf{(Theorem~\ref{mainTh})}   \emph{Let $M= WH$  be a normalized $r$-separable matrix where $W$ is $\kappa$-robustly conical with $\kappa > 0$. 
  Let also $\tilde{M} = M+N$  with $||N||_{1} \leq \epsilon$. If 
   \[ 
  \epsilon < \frac{\omega \kappa}{99(r+1)}, 
  \]  
  where $\omega = \min_{i \neq j} ||W(:,i)-W(:,j)||_1$ (note that $\omega\geq \kappa$),   then Algorithm~\ref{postpro} extracts a matrix $\tilde{W}$ such that 
  \[
  ||W - \tilde{W}(:,P)||_{1} \leq 49(r+1) \frac{\epsilon}{\kappa} + 2\epsilon, \quad \text{   for some permutation $P$.} 
  \]}
  Because the necessary condition $\epsilon <  \frac{\kappa}{r-1}$ also applies to Algorithm~\ref{postpro}, the bound for~$\epsilon$ of Theorem~\ref{mainTh} is tight up to a factor $\omega$ (and some constant multiplicative factor). 
  Moreover, because of the necessary condition on $\delta$ (see above), 
  %Theorem~\ref{lbdelta}, the bound for $\delta$ of 
  Theorem~\ref{mainTh} is tight up to a factor $r$ (and some constant multiplicative factor). 
 Finally, we show that it is necessary  for Proposition~\ref{prop2} to hold that $\epsilon \leq \frac{\kappa}{(r-1)^2}$ for any $\delta < \kappa + \epsilon$  (Theorem~\ref{th7}) which demonstrates that Hottopixx cannot achieve a better bound than Algorithm~\ref{postpro}. We also compare Algorithm~\ref{postpro} with the algorithm of Arora et al.~\cite{AGKM11} in Section~\ref{aror}.  

In the last part of the paper (Section~\ref{ne}), we illustrate our results on some synthetic datasets: we show that the post-processing makes Hottopixx more robust to noise, and able to deal with duplicates and near duplicates of the columns of $W$.   \vspace{0.2cm}

\subsection{Notation} 

The set of  $m$-by-$n$ real matrices is denoted $\mathbb{R}^{m \times n}$; for $A \in \mathbb{R}^{m \times n}$, we denote the $j$th column of $A$ by $A(:,j)$, the $i$th row of $A$ by $A(i,:)$, and the entry at position $(i,j)$ by $A(i,j)$; for $b \in \mathbb{R}^{m \times 1} = \mathbb{R}^{m}$, we denote the $i$th entry of $b$ by $b(i)$.   
Notation $A(\mathcal{I},\mathcal{J})$ refers to the submatrix of $A$ with row and column indices respectively in $\mathcal{I}$ and $\mathcal{J}$. 
The matrix $A^T$ is the transpose of $A$.  
The $\ell_1$-norm $||.||_1$ of a vector is defined as $||b||_1 = \sum_i |b(i)|$ and of a matrix as $||A||_1 = \max_j ||A(:,j)||_1$. The $\ell_\infty$-norm $||.||_\infty$ of a vector is defined as $||b||_\infty = \max_i |b(i)|$. 
We will denote by $E_{n}$ the $n$-by-$n$ matrix of all-ones, $0_{m \times n}$ the $m$-by-$n$ the matrix of all-zeros, and $I_n$ the $n$-by-$n$ identity matrix. We will also denote $e_i$ the $i$th column of the identity matrix, $e$ the all-one vector and $0$ the all-zero vector; their dimensions will be clear from the context. 
The vector of the diagonal entries of a matrix $A$ is denoted $\diag(A)$ while its trace is denoted $\tr(A) = e^T \diag(A)$. For a set $\mathcal{K}$, $|\mathcal{K}|$ denotes its cardinality. \vspace{0.2cm}

\section{Analysis without Duplicates nor Near Duplicates} \label{ip33} 
%of Proposition~\ref{prop1}} \label{ip33}

In this section, we focus on Proposition~\ref{prop1}: we show that $\epsilon \leq \frac{\kappa (1- \beta)}{9(r+1)}$ is sufficient for Proposition~\ref{prop1} to hold (Theorem~\ref{Th2}), while $\epsilon <  \frac{\kappa (1-\beta)}{(r+1)(1-\beta)+1}$ is necessary for any $r\geq 3$ and $\beta < 1$ (Theorem~\ref{nc1}). 
%These results will be useful when analyzing the general case in Section~\ref{ip2}. 

\begin{lemma} \label{lem2}
Suppose $\tilde{M} = M+N$ where $M$ is normalized and $||N||_{1} \leq \epsilon < 1$, and suppose $X$ is a feasible solution of \eqref{rechtLP}. Then, for all $1 \leq j \leq n$,  
\[
||X(:,j)||_1 \leq   1 +  \frac{4\epsilon}{1-\epsilon} %\frac{1+\epsilon}{1-\epsilon} 
\quad 
\text{ and } 
\quad 
||M(:,j) - M X(:,j)||_1 \leq \frac{4\epsilon}{1-\epsilon}. 
\]
\end{lemma}
\begin{proof}
For all $1 \leq j \leq n$, 
\[
1 - \epsilon \leq ||M(:,j)||_1 - ||N(:,j)||_1 \leq ||M(:,j)+N(:,j)||_1 = ||\tilde{M}(:,j)||_1, 
\] 
from which we obtain 
\begin{align*}
2 \epsilon  \geq  ||\tilde{M}(:,j) - \tilde{M} X(:,j)||_1  
& \geq   ||\tilde{M} X(:,j)||_1 - ||\tilde{M}(:,j)||_1 \\
& \geq ||{M} X(:,j)||_1 - ||N X(:,j)||_1 - (1+\epsilon) \\ 
& \geq   ||X(:,j)||_1 - \epsilon ||X(:,j)||_1 - 1 - \epsilon, 
%& \geq   ||\tilde{M}||_{1} (||X(:,j)||_1 - 1) \\ 
%&  \geq (1-\epsilon) (||X(:,j)||_1  - 1) , 
\end{align*} 
since the columns of $M$ sum to one and $M$ and $X$ are nonnegative. This implies that 
$||X(:,j)||_1 \leq  \frac{1+3\epsilon}{1-\epsilon} =  1+ \frac{4\epsilon}{1-\epsilon}$, and  
$||NX(:,j)||_1 \leq ||N||_{1} ||X(:,j)||_1 \leq \epsilon \left( \frac{1+3\epsilon}{1-\epsilon} \right)$. 
We then have 
\begin{align*}
2 \epsilon  \geq ||\tilde{M}(:,j) - \tilde{M} X(:,j)||_1 
&  = ||{M}(:,j) + N(:,j) - ({M}+N) X(:,j)||_1 \\
&  \geq ||{M}(:,j) - M X(:,j)||_1 - \epsilon - \epsilon \left(  \frac{1+3\epsilon}{1-\epsilon} \right) , 
%&  \geq  ||{W}(:,k) - WH X(:,j)||_1 - 2 \epsilon \left( \frac{1}{1-\epsilon} \right). 
\end{align*}
hence $||{M}(:,j) - M X(:,j)||_1 \leq 3 \epsilon + \epsilon \left(  \frac{1+3\epsilon}{1-\epsilon} \right) = \frac{4\epsilon}{1-\epsilon}$. \vspace{0.2cm} 
\end{proof}

\begin{remark}
Note that a sum-to-one constraint on the columns of $X$ could be added to the model~\eqref{rechtLP} while keeping linearity, and would make the analysis simpler. In fact, we would directly have $||X(:,j)||_1 = 1$ and $||M(:,j) - M X(:,j)||_1 \leq {4\epsilon}$ for all $j$, 
and the error bounds from Theorems~\ref{Th2} and \ref{mainTh} could be slightly improved. 
However, we stick in this paper with the original formulation proposed in \cite{BRRT12}. \vspace{0.2cm}
\end{remark}

\begin{lemma} \label{lem3}
Let $\tilde{M} = M+N$ where $M$ is normalized, admits a rank-$r$ separable factorization $WH$ where $W$ is $\kappa$-robustly conical with $\kappa > 0$, and $||N||_{1} \leq \epsilon < 1$, and has the form \eqref{sepM} with 
$\max_{i,j} H'_{ij} \leq \beta < 1$. Let also $X$ be any feasible solution of \eqref{rechtLP}, then 
\[
X(j,j) \geq  1 - \frac{8 \epsilon}{\kappa (1-\beta) (1-\epsilon)}, % \left( \frac{3-\epsilon}{1-\epsilon} \right) 
\]
for all $j$ such that $M(:,j) = W(:,k)$ for some $1 \leq k \leq r$.  
\end{lemma}
\begin{proof} 
Let $\mathcal{K}$ be the set of $r$ indices such that $M(:,\mathcal{K}) = W$. Let also $1 \leq k \leq r$ and denote $j = \mathcal{K}(k)$ so that $M(:,j) = W(:,k)$. 
By Lemma~\ref{lem2}, 
\begin{equation} \label{ubo}
||W(:,k) - WH X(:,j)||_1 \leq \frac{4\epsilon}{1-\epsilon}. 
\end{equation}
%Let us prove a lower bound for $||{W}(:,k) - WH X(:,i)||_1$. 
Since $H(k,j) = 1$, 
\begin{align*}
WH X(:,j) & = W(:,k) H(k,:) X(:,j)  + W(:,\mathcal{R}) H(\mathcal{R},:)X(:,j) \\
%& = W(:,k) X(j,j) + W(:,k) H(k,\mathcal{J}) X(\mathcal{J},j)  + W(:,\mathcal{R}) y \\
& = W(:,k) \Big( X(j,j) +  H(k,\mathcal{J})X(\mathcal{J},j) \Big)  + W(:,\mathcal{R}) y, 
\end{align*} 
where $\mathcal{R} = \{1,2,\dots,r\} \backslash \{k\}$, $\mathcal{J}= \{1,2,\dots,n\} \backslash \{j\}$ and $y = H(\mathcal{R},:)X(:,j) \geq 0$. We have  
\begin{equation} \label{etaa}
\eta = X(j,j) +  H(k,\mathcal{J}) X(\mathcal{J},j) \leq X(j,j) + \beta \left(1+ \frac{4\epsilon}{1-\epsilon} -X(j,j)\right), 
\end{equation}
since $||H(k,\mathcal{J})||_{\infty} \leq \beta$ and $||X(:,j)||_1 \leq  1 + \frac{4 \epsilon}{1-\epsilon}$ (Lemma~\ref{lem2}). Hence 
\begin{equation} \label{eta}
||W(:,k) - WH X(:,j)||_1 \geq (1-\eta) \left\|W(:,k) - W(:,\mathcal{R}) \frac{y}{1-\eta}\right\|_1 \geq (1-\eta) \kappa. 
\end{equation}
Combining Equations~\eqref{ubo}, \eqref{etaa} and \eqref{eta},  we obtain %and using Lemma~\ref{trick}, 
\[
 1 - \left( X(j,j) + \beta \left(  1 + \frac{4 \epsilon}{1-\epsilon} -X(j,j)\right) \right) 
 \leq \frac{4\epsilon}{\kappa (1-\epsilon)} , 
\]
%since $||X(:,j)||_1 \leq 1 + \frac{4 \epsilon}{1-\epsilon}$. 
which gives, using the fact that $\kappa, \beta \leq 1$, 
\begin{align*}
 X(j,j) 
& \geq  1 - \frac{8 \epsilon}{\kappa (1-\beta) (1-\epsilon)}. 
%& \geq  1 -   \frac{8 \epsilon}{\kappa(1-\epsilon)}
\end{align*}
\end{proof}

\begin{theorem} \label{Th2} 
It is sufficient for Proposition~\ref{prop1} to hold that 
\[
\epsilon \leq %f(\kappa,\beta,r) = 
\frac{\kappa (1- \beta)}{9 (r+1)}.  
\] 
\end{theorem}  
\begin{proof} 
If $\epsilon = 0$, the proof is given in \cite[Th.\@ 3.1]{BRRT12}: for each $1 \leq k \leq r$, there exists a unique $j \in \{1,2,\dots,n\}$ such that $M(:,j) = W(:,k)$ and $X(j,j) = 1$ (this follows easily from the fact that the entries of $p$ are distinct). (Note that, in the noiseless case when $\epsilon = 0$, duplicates and near duplicates are allowed in the dataset since $\beta$ can be equal to one.) 
Otherwise $\epsilon > 0$ and $\beta < 1$. Let $X$ be a feasible solution of \eqref{rechtLP} (which always exists since  the feasible set of \eqref{rechtLP} is non-empty). If we prove that the $r$ diagonal entries of $X$ corresponding to the columns of $W$ are larger than all the other ones (because $\beta < 1$, the are no duplicates of the columns of $W$ in the dataset), then we are done. 
In fact, these columns will then be identified by Algorithm~\ref{balgo} and we will have 
$|| W - \tilde{W}(:,P)||_{1} \leq \epsilon$ for some permutation~$P$. (Notice that we do not need an optimal solution: any feasible solution identifies the columns of $W$.) 
%hence $||\tilde{M} - \tilde{W} \tilde{H}||_{1} \leq 2 \epsilon$ (Lemma~\ref{lem1}). 

Let $\mathcal{K}$ be the set of $r$ indices such that $M(:,\mathcal{K}) = W$. Assume that 
\begin{equation} \label{rtp}
X(k,k) > \frac{r}{r+1} \qquad \text{ for all $k \in \mathcal{K}$.} 
\end{equation}
Since $\tr(X) = r$ and $X \geq 0$, we have 
\[
\sum_{j \notin \mathcal{K}} X(j,j) 
= r - \sum_{k \in \mathcal{K}} X(k,k) 
< r - r \frac{r}{r+1} 
= \frac{r}{r+1} 
< X(k,k) \quad  \text{ for all $k \in \mathcal{K}$}, 
\]
implying that $X(j,j) < X(k,k)$ for all $k \in \mathcal{K}, j \notin \mathcal{K}$ which gives the result. It remains to show that \eqref{rtp} holds.  
By Lemma~\ref{lem3}, 
\begin{align*}
X(k,k) & \geq  1 - \frac{8\epsilon}{\kappa (1-\beta)(1-\epsilon)}
 %\geq 1 - \frac{83\epsilon}{20 \kappa (1-\beta)}
 > 1 - \frac{9 \epsilon}{\kappa (1-\beta)}, 
\end{align*}
since $\frac{8}{1-\epsilon} < 9$ for any $\epsilon \leq \frac{\kappa (1-\beta)}{9(r+1)} \leq \frac{1}{18}$ as $\kappa > 0$, $\beta < 1$ and $r \geq 1$. Finally, for $\epsilon \leq \frac{(1-\beta) \kappa}{9 (r+1)}$, $X(i,i) > \frac{r}{r+1}$ and the proof is complete. 
\end{proof}

\begin{remark}
%It is interesting to observe that 
The proof of Theorem~\ref{Th2} actually does not make use of the constraints $X(i,j) \leq X(i,i)$ for all $i,j$. 
The reason is that the assumption $\max_{i,j} H'_{ij} \leq \beta < 1$ implies that there is no duplicate of the columns of $W$ in the dataset (if $\beta = 1$, $\epsilon = 0$ and Algorithm~\ref{balgo} is guaranteed to work \cite[Th.\@ 3.1]{BRRT12}). This implies that for being able to reconstruct sufficiently well each column of $W$, the corresponding diagonal entry of $X$ must be large independently of the other entries of the corresponding column of $X$.  

Therefore, in case there is no duplicate in the dataset (or some some pre-processing has been used to remove them), these constraints can be discarded (a similar observation was made in \cite{BRRT12}). Moreover, since Theorem~\ref{Th2} only requires feasibility in that case, any feasible solution of the corresponding relaxed linear program will correctly identify the columns of $W$.  \vspace{0.2cm}
%The reason is that the assumption $||H'||_{\infty} \leq \beta < 1$ implies that there is not duplicate of the columns of $W$ in the dataset. Therefore, to being able to reconstruct sufficiently well each column of $W$, we need to extract them. 
\end{remark}

\begin{theorem} \label{nc1} 
For Proposition~\ref{prop1} to hold when $r \geq 3$ and $\beta < 1$, %$for any $\alpha$, $\beta$, $r \geq 3$, 
it is necessary that 
\begin{equation} \label{ub1}
\epsilon 
%\leq \frac{\kappa (1-\beta)}{(r+1)(1-\beta) + 3 - \beta}  
< \frac{\kappa (1-\beta)}{(r-1)(1-\beta) + 1} . 
%\leq \frac{\kappa (1-\beta)}{r + 1} . 
\end{equation} 
\end{theorem}  
\begin{proof}
See Appendix~\ref{appb}. \vspace{0.2cm}
\end{proof}  
  
  Theorem~\ref{nc1} shows that the sufficient condition derived in Theorem~\ref{Th2} is close to being tight. In particular, if $r$ is assumed to be bounded above, then it is tight up to some constant multiplicative factor (in practice $r$ is often assumed to be small). 
  We believe it is possible to improve the bound of Theorem~\ref{Th2} to match the one of 
  %\footnote{We conjecture that the gap can be closed by deriving a tighter upper bound in Equation~\eqref{etaa}.} 
  Theorem~\ref{nc1} (up to some constant multiplicative factor). 
  Unfortunately, we were not able to derive such a sufficient condition; this is a topic for further research. \vspace{0.2cm} 
  
    \begin{remark}[Cases $r=1,2$] Theorem~\ref{nc1} does not apply when $r=1,2$ because: 
  \begin{itemize}
  
  \item The rank-one separable NMF problem is trivial. In fact, if $M$ admits a rank-one separable factorization $wh^T$ and $\tilde{M} = M+N$ with $||N||_{1} \leq \epsilon$, then $||\tilde{M}(:,j) - w ||_{1} \leq \epsilon$ for all $j$. 
  
  \item The rank-two case is particular because it is not possible to construct very bad instances. In fact, all rank-two separable NMF problems are essentially equivalent to each other because the columns of $M$ belong to the line segment $[W(:,1),W(:,2)]$. \vspace{0.2cm}  
  %hence $\epsilon$ essentially has to be of the order $\kappa (1-\beta)$ to 
  \end{itemize}
  \end{remark}

To conclude this section, we provide a necessary condition for Proposition~\ref{prop2}:  

  \begin{corollary} \label{cor2}
  For Proposition~\ref{prop2} to hold for any $\delta < \frac{\kappa}{2}$ and $r \geq 3$, it is necessary that 
  \[
  \epsilon < \frac{\kappa}{r-1} , %\qquad \text{ implying } \quad g(\alpha,r) \leq  \frac{\alpha}{r-1}. 
  \]
   \end{corollary}
   \begin{proof}
   In fact, 
   \[
   \epsilon < \frac{\kappa (1-\beta)}{(1-\beta)(r-1)+1} \leq \frac{\kappa (1-\beta)}{(1-\beta)(r-1)} = \frac{\kappa}{r-1} , %\leq  \frac{\alpha}{r-1}, 
   \]
     %For any normalized $r$-separable matrix $M$, we have $\beta = ||H'||_{\infty} \geq \frac{1}{r}$ (by the pigeon-hole principle). 
     while the matrix $\tilde{M}=WH+N$ constructed in the proof of Theorem~\ref{nc1} satisfies $||W - \tilde{W}(:,P)||_{1} \geq  \frac{r-2}{r-1}  \kappa \geq \frac{\kappa}{2}$ where $\tilde{W}$ is the matrix extracted by Algorithm~\ref{balgo2} and $P$ is any permutation. %for any permutation $P$. %(Because we assume the columns of $M$ to sum to one, $\delta = 2$)
   \end{proof}

\section{Dealing with Duplicates and Near Duplicates using Post-Processing} \label{ip2}

In this section, we investigate Proposition~\ref{prop2} and propose a variant of Hottopixx (see Algorithm~\ref{postpro}) which is provably robust for \emph{any noisy separable matrix}. 
In Section~\ref{pnc}, we present a simple necessary condition for Proposition~\ref{prop2} to hold. 
In Section~\ref{ci}, we show that, for each column of $W$, there is a subset of the columns of $\tilde{M}$ close to that column of $W$ such that the sum of the corresponding diagonal entries of any feasible solution $X$ of \eqref{rechtLP} is larger than $\frac{r}{r+1}$. 
Therefore, using an appropriate post-processing of the solution $X$ of \eqref{rechtLP} (see Algorithm~\ref{postpro}), 
we can approximately recover the columns of $W$, given that the noise level $\epsilon$ is smaller than some upper bound. 
%give a sufficient condition for Proposition~\ref{prop2} to hold. 
In Section~\ref{repar}, we show that Hottopixx (Algorithm~\ref{balgo2}) cannot achieve this bound which proves that Algorithm~\ref{postpro} is more robust. Finally, we compare Algorithm~\ref{balgo2} with the algorithm of Arora et al.\@~\cite{AGKM11}  in Section~\ref{aror}.

\subsection{Preliminary Necessary Conditions} \label{pnc}

Recall the aim is to identifying, among the columns of $\tilde{M}$, $r$ columns gathered in the matrix $\tilde{W}$ in such a way that \mbox{$||W - \tilde{W}(:,P)||_{1}$} $\leq \delta$ for some permutation $P$ and some $\delta \geq 0$. 
Since $||W||_{1} = 1$, we will assume that  $\delta < 1$ otherwise the separable NMF problem is trivial since the solution $\tilde{W} = 0$ gives the result. It actually makes sense to impose $\delta < {\kappa} \leq \alpha \leq 1$: this guarantees for a solution $\tilde{W}$ to have distinct columns since two columns of $W$ can potentially be at distance $\kappa$; for example with 
 \[
 W = \left( \begin{array}{cc} 
 \frac{\kappa}{2} & 0 \\ 
 0 & \frac{\kappa}{2} \\
 1-\frac{\kappa}{2} &  1-\frac{\kappa}{2}
 \end{array} \right), 
 \]
extracting twice the first column would give the result with $\delta = \kappa$, which is not desirable. 
Moreover, as shown in Section~\ref{hottop}, it is necessary that $\epsilon < \frac{\alpha}{2} \leq \kappa$ for \emph{any}  separable NMF algorithm to being able to extract approximately the columns of $W$.  

\begin{theorem} \label{lbdelta}
For any $0 \leq \epsilon < \frac{\alpha}{2}$, it is necessary that $\delta \geq \left(3 \frac{\epsilon}{\alpha} + \frac{3}{2} \epsilon \right)$ for Proposition~\ref{prop2} to hold. 
\end{theorem}
\begin{proof} 
Let us consider $\tilde{M} = M + N = WH + N$ where
\[
W =  
\left( \begin{array}{ccc}
1 & 0 & \frac{1}{2} - \frac{\alpha}{4} \\
0 & 1 & \frac{1}{2} - \frac{\alpha}{4} \\
0 & 0 & \frac{\alpha}{2} 
%0 & 0 & 0
\end{array} \right), 
H = \left( \begin{array}{ccccc}
1 & 0 & 0  &  1-\lambda & 0 \\
0 & 1 & 0  & 0 & 1-\lambda\\
0 & 0 & 1  & \lambda &  \lambda  
\end{array} \right),  
\] 
\[ \text{ and }  
N = \left( \begin{array}{ccccc}
0 & 0 & \frac{\epsilon}{4} & 0 & 0\\
0 & 0 &\frac{\epsilon}{4}  & 0 & 0\\
0 & 0  & -\frac{\epsilon}{2} & 0 & 0
\end{array} \right), 
\] 
where $W$ is $\alpha$-robustly simplicial (and $\frac{\alpha}{2}$-robustly conical), and where $\lambda$ is such that the middle point between $M(:,4)$ and $M(:,5)$ is $\left( \tilde{M}(:,3) + 2 N(:,3) \right)$, that is, 
\begin{align*}
\tilde{M}(:,3) + 2 N(:,3) & = \left( \begin{array}{c}
\frac{1}{2} - \frac{\alpha}{4} + \frac{3\epsilon}{4}  \\ 
\frac{1}{2} - \frac{\alpha}{4} + \frac{3\epsilon}{4}  \\ 
%\frac{1}{2} - \frac{\alpha}{4} + \frac{\epsilon}{2} \\ 
 \frac{\alpha}{2} - \frac{3\epsilon}{2}  
\end{array} \right) = 
\left( \begin{array}{c}
 \frac{1}{2} - \frac{\lambda \alpha}{4}  \\
  \frac{1}{2} - \frac{\lambda \alpha}{4}  \\   
%\frac{1}{2} - \frac{\alpha}{4} + \frac{\epsilon}{2} \\ 
\frac{\lambda \alpha}{2} \\ 
\end{array} \right) = \frac{1}{2} (M(:,4)+M(:,5)) , 
\end{align*}
which requires $\lambda = 1-3\frac{\epsilon}{\alpha} \geq 0$. 
Let $p = (-K, -K, K^2,-1,0)^T$ for any $K$ sufficiently large. It can be checked that 
\[
X = \left( \begin{array}{ccccc}
1 & 0  & 0   & \mu   & 0\\
0 & 1  & 0   & 0   & \mu\\ 
0 & 0  & 0   & 0   & 0\\
0 & 0  & 0.5 & 0.5 & 0.5-\mu\\
0 & 0  & 0.5 & 0.5-\mu   & 0.5
\end{array} \right) \qquad \text{ where } \mu = \frac{1-\lambda}{2-\lambda}, 
\]
is a feasible solution of \eqref{rechtLP}. By Lemma~\ref{lplem}, there exists $K$ sufficiently large such that $X^*(3,3) = 0$ for any optimal solution $X^*$. Using Lemma~\ref{lplem} again we have $X^*(1,1) = X^*(2,2) = 1$ for any optimal solution $X^*$ for $K$ sufficiently large. 
Hence, for $K$ sufficiently large, the third column of $M$ will not be extracted and the fourth or fifth will be, hence  
% By symmetry, whether the fourth or fifth column is extracted does not influence the error, so let $\tilde{W} = M(:,[1 \, 2 \, 4])$ without loss of generality: 
\begin{align*}
||\tilde{W} - W||_{1} = ||\tilde{W}(:,3) - W(:,3)||_1 & = ||M(:,4) - W(:,3)||_1 = ||M(:,5) - W(:,3)||_1 \\
& = ||(1-\lambda) W(:,1) - (1-\lambda) W(:,3)||_1 \\
& = 3\frac{\epsilon}{\alpha} \left\| W(:,1) - W(:,3) \right\|_1 \\
& =  3\frac{\epsilon}{\alpha} \left( {1} + \frac{\alpha}{2} \right) = 3 \frac{\epsilon}{\alpha} + \frac{3}{2} \epsilon. 
\end{align*} 
\end{proof}

Using the same construction\footnote{A Matlab code is available at \url{https://sites.google.com/site/nicolasgillis/code} and contains this construction, along with the one of Theorem~\ref{lbdelta}.} as in Theorem~\ref{lbdelta} but taking $\lambda = 1-\frac{\epsilon}{\alpha}$, we have 
\[
\tilde{M}(:,3) = W(:,3) + N(:,3) = \frac{1}{2} \left( M(:,4)+M(:,5) \right), 
\]
for which $||\tilde{W}(:,P) - W||_{1} \geq \frac{\epsilon}{\alpha} + \frac{\epsilon}{2}$ for any permutation $P$, where $\tilde{W}$ is the matrix extracted by Hottopixx.  
We notice that the corresponding matrix $\tilde{M}$ can also be obtained from a 4-separable matrix $M_4 = W_4H_4$ where 
\[
W_4 =  \left( \begin{array}{ccc}
M(:,[1 \, 2])  & M(:,4)-v   & M(:,5)-v \\
\end{array} \right), 
H_4 = \left( \begin{array}{ccccc}
1 & 0 & 0  &  0 & 0 \\
0 & 1 & 0  & 0 & 0\\
0 & 0 & 0.5  & 1 &   0\\ 
0 & 0 & 0.5  & 0 & 1\\
\end{array} \right), 
\]
$v = (\epsilon/4, \epsilon/4, -\epsilon/2)^T$, and  
\[
N_4 = \left( \begin{array}{cccc}
0_{3 \times 2} & v & v & v \\ 
\end{array} \right) , 
\] 
and we have $\tilde{M} = WH + N = W_4 H_4 + N_4 = \tilde{M}_4$. 
Therefore, \emph{no algorithm to which only the noisy separable matrix $\tilde{M}$ and the noise level $\epsilon$ are given as input can approximately extract the columns of $W$ among the columns of $M$ with error smaller than $\mathcal{O}\left(\frac{\epsilon}{\alpha}\right)$. } 
In fact, the matrix $\tilde{M}$ above has two solutions to the noisy separable NMF problem and there is no way to discriminate between them (the original matrix could be 3- or 4- separable): 
\begin{itemize} 
\item If the algorithm returns a matrix $\tilde{W}$ with three columns, then if the original matrix was $M_4$ we have $\max_{1 \leq j \leq 4} \min_{1 \leq k \leq 3} || W_4(:,j) - \tilde{W}(:,k)||_1 \geq \frac{\epsilon}{\alpha}$. 
\item Similarly, if the algorithm returns a matrix $\tilde{W}$ with four columns, then 
	\begin{itemize} 
	\item if the third column is not extracted and the original matrix was ${M}$, we have  \[\max_{1 \leq j \leq 3} \min_{1 \leq k \leq 4} || W(:,j) - \tilde{W}(:,k)||_1 \geq \frac{\epsilon}{\alpha} , \quad \text{while} \] 
	\item if the third column is extracted and the original matrix was ${M}_4$, we have \[\max_{1 \leq j \leq 4} \min_{1 \leq k \leq 4} || W_4(:,j) - \tilde{W}(:,k)||_1 \geq \frac{\epsilon}{\alpha}.\] 
\end{itemize} 
\end{itemize} 
The reason is that the distance between each pair of columns of $M$ is at least $\frac{\epsilon}{\alpha}$.  \vspace{0.1cm}

Note that the algorithm of Arora et al.\@~\cite{AGKM11} achieves this optimal error bound $\mathcal{O}\left(\frac{\epsilon}{\alpha}\right)$; see Theorem~\ref{tharor} in Section~\ref{aror}. 
However, \emph{it requires the parameter $\alpha$ as an input} so that the construction above does not prove their algorithm is optimal up to some constant multiplicative factor. 
%its optimality (up to a constant multiplicative factor). 
In fact, for the 3-separable matrix $M$, $W$ is $\alpha$-robustly simplicial while, for the 4-separable matrix $M_4$, $W_4$ is $\alpha'$-robustly simplicial with $\alpha' \leq 2 \frac{\epsilon}{\alpha} =$ \mbox{$||W_4(:,3) - W_4(:,4)||_1$}.   
It is possible to adjust the construction so that $W$ and $W_4$ have the same condition number, proving that the algorithm of Arora et al.~\cite{AGKM11} is optimal. It suffices to add a row and a column to the input matrices as follows 
\[
M' = \left( \begin{array}{cc}
M & \left(1-\frac{\alpha}{\epsilon}\right) We \\ 
0 & 1-\frac{\alpha}{\epsilon}
\end{array} \right) , \quad \text{ and } \quad 
 M'_4 = \left( \begin{array}{cc}
M_4 & \left(1-\frac{\alpha}{\epsilon}\right) W_4 e \\ 
0 & 1-\frac{\alpha}{\epsilon}
\end{array} \right), 
\]
(and updating $W, W_4, H$ and $H_4$ accordingly) where $M'$ is 4-separable with conditioning $\frac{\epsilon}{\alpha}$ while $M'_4$ is 5-separable with the same conditioning.

\subsection{Cluster Identification} \label{ci}

We now prove that there is a cluster of columns of $\tilde{M}$ around each column of $W$ for which the sum of the corresponding diagonal entries of any feasible solution $X$ of \eqref{rechtLP} is large. 
%The proof is very similar to the proof of Theorem~\ref{Th2}. 
  More formally, defining the clusters around the columns of $W$ as 
   \begin{equation} \label{clusters} 
  \Omega_k^{\rho} = \Big\{ j \ \Big| \ ||\tilde{M}(:,j) - W(:,k)||_1 \leq \rho  \Big\} \quad 1 \leq k \leq r, 
  \end{equation}
  we are going to prove that $c_k = \sum_{j \in \Omega_k^{\rho}} X(j,j)$ is large for any feasible solution $X$ of \eqref{rechtLP}, given that $\epsilon$ is sufficiently small.

\begin{lemma} \label{lem4} Let $W \in \mathbb{R}^{m \times r}_+$ %be $\alpha$-robustly canonical and 
have its columns sum to one, and let $h \in \Delta^m$. Then, denoting $k = \argmax_{1 \leq i \leq r} h(i)$, we have 
\[
||h||_{\infty} = h(k) \geq 1-\frac{\rho}{2} 
\quad \Rightarrow \quad 
 ||W(:,k) - Wh||_1 \leq \rho.   %\quad  \text{ for some $1 \leq k \leq r$.} 
\]
  \end{lemma}
  \begin{proof} %We have that $h_k \geq 1-\frac{\delta}{2}$ for some $1 \leq k \leq r$. 
  Let us denote denote $\mathcal{R} = \{1,2,\dots,r\} \backslash \{k\}$,   we have  
  \begin{align*}
 % \min_{h \dots} 
 ||W(:,k) - Wh||_1 
  & = ||(1-h(k)) W(:,k) - W(:,\mathcal{R})h(\mathcal{R})||_1 \\
%  %& = (1-h_k) \left\| W(:,k) - W(:,K)\frac{h_k}{(1-h_k)}\right\|_1 \\
%   & = \dots \\
%   & \leq (1-h_k) \alpha = \dots \\
  & \leq (1-h(k)) ||W(:,k)||_1 + (||h||_1 - h_k) ||W(:,\mathcal{R})||_{1} \\
  & \leq 2 (1-h(k)) \leq \rho. 
  %= (1-x(1)) || w - W y'||_1 > \alpha \frac{\delta}{\alpha} = \delta. 
  \end{align*}
    \end{proof}

\begin{lemma} \label{lem5}
Let $M=WH$  be a normalized $r$-separable matrix where $W$ is \mbox{$\kappa$-robustly} conical with $\kappa > 0$. 
 Let also $\tilde{M} = M+N$  where $||N||_{1} \leq \epsilon < 1$, and $X$ be a feasible solution of \eqref{rechtLP}. 
Then, the total weight $c_k = \sum_{j \in \Omega_k^{\rho}} X(j,j)$ assigned to the columns of $\tilde{M}$ in $\Omega_k^{\rho}$ defined in \eqref{clusters} satisfies 
\[
c_k \geq %1 - \frac{7 \epsilon}{\alpha (1-\beta)} = 
1 - \frac{16 \epsilon}{\kappa \rho (1-\epsilon)}  \qquad \text{ for all } 1 \leq k \leq r. 
\] 
\end{lemma}  
  \begin{proof} 
    Let $1 \leq k \leq r$ and $\mathcal{R} = \{1,2,\dots,r\} \backslash \{ k\}$, and let us denote the indices corresponding to the columns of $\tilde{M}$ not in $\Omega_k^{\rho}$ as 
\[
\bar{\Omega}_k^{\rho} = \{ 1,2,\dots,n \} \backslash \Omega_k^{\rho}. 
\]
Let also $j$ be such that $W(:,k) = M(:,j)$. 
By Lemma~\ref{lem4}, $\max_{j \in \bar{\Omega}^{\rho}_k} ||H(:,j)||_{\infty} < 1 - \frac{\rho}{2} = \beta$. 
  The rest of the proof is similar to that of Lemma~\ref{lem3}. By Lemma~\ref{lem2}, $||W(:,k) - WH X(:,j)||_1 \leq  \frac{4\epsilon}{1-\epsilon}$ and $||X(:,j)||_1 \leq 1+ \frac{4 \epsilon}{1-\epsilon}$. 
 We have  
\begin{align*}
WH X(:,j) & =  W(:,k) H(k,:) X(:,j)  + W(:,\mathcal{R}) H(\mathcal{R},:)X(:,j) \\
%& =  W(:,k) H(k,:) X(:,j)  + W(:,{K}) y \\ 
 & = W(:,k) \Big(   H(k,{\Omega}_k^{\rho})X({\Omega}_k^{\rho},j)+  H(k,\bar{\Omega}_k^{\rho})X(\bar{\Omega}_k^{\rho},j) \Big)  + W(:,\mathcal{R}) y, 
\end{align*}
where $y = H(\mathcal{R},:)X(:,j) \geq 0$, and  
  \begin{align*}
  \eta &  =  H(k,\Omega_k^{\rho}) X(\Omega_k^{\rho},j) + H(k,\bar{\Omega}_i^{\rho}) X(\bar{\Omega}_i^{\rho},j) \\ 
 & \leq  ||X(\Omega_k^{\rho},j)||_1 + \beta (||X(:,j)||_1 - ||X(\Omega_k^{\rho},j)||_1)  
  \leq  c_k + \beta \left( 1 + \frac{4\epsilon}{1-\epsilon} - c_k \right) .
 \end{align*}
 The first inequality follows from $H(i,j) \leq 1$ for all $i,j$ and $||H(k,\bar{\Omega}_k^{\rho})||_{\infty} \leq \beta$; the second from $X(i,j) \leq X(i,i)$ for all $i, j$ (hence $c_k \geq ||X(\Omega_k^{\rho},j)||_1$), and $\beta \leq 1$. 
 Finally, $(1-\eta) \kappa \leq$  \mbox{$||W(:,k) - WH X(:,j)||_1$} $\leq \frac{4\epsilon}{1-\epsilon}$ leading to 
$c_k = \sum_{j \in \Omega_k^{\rho}} X(j,j) 
 \geq 1 - \frac{8 \epsilon}{\kappa (1-\beta)(1-\epsilon)}  
 = 1 - \frac{16 \epsilon}{\kappa \rho (1-\epsilon)}$. \vspace{0.2cm}
  \end{proof}  
    
 If we can guarantee that $c_k > \frac{r}{r+1}$ for all $1 \leq k \leq r$, 
  then the sum of the diagonal entries of $X$ corresponding to columns of $\tilde{M}$ not in any $\Omega_k^{\rho}$ will be smaller than $\frac{r}{r+1}$.  Therefore, if instead of picking the $r$ largest diagonal entries of $X$, we cluster the diagonal entries of $X$ depending on the distances between the corresponding columns of $\tilde{M}$, we should be able to identifying the columns of $W$ approximately; see Algorithm~\ref{postpro}.

\algsetup{indent=2em}
\begin{algorithm}[ht!]
\caption{Extracting Columns of a Separable Matrix by Linear Programming and Clustering \label{postpro}}
\begin{algorithmic}[1] 
\REQUIRE A $r$-separable matrix $\tilde{M} = WH + N$ with $W$ $\kappa$-robustly conical, the noise level $||N||_{1} \leq \epsilon$ and  the factorization rank $r$.  
\ENSURE A matrix $\tilde{W}$ such that $||\tilde{W}(:,P)-{W}||_{1}$ is small for some permutation $P$. \medskip 
\STATE Compute the optimal solution $X$ of \eqref{rechtLP}. 
%\STATE Initialize $S_i = \{ i \}$ $1 \leq i \leq n$; $w(i) = X(i,i)$; 
\STATE Initialize $\mathcal{K} = \{ k \ | \ X(k,k) > \frac{r}{r+1} \}$ and $\nu = 2 \epsilon$. 
%\STATE $D_{ij} = ||M(:,i)-M(:,j)||_1$ for $1 \leq i, j \leq n$.
%\STATE ; 
\WHILE{$|\mathcal{K}| < r$ and $\nu \leq 2 ||\tilde{M}||_1$} 
\STATE Compute $\mathcal{K}$ with Algorithm~\ref{cluster} using input $m_j = \tilde{M}(:,j)$ $1 \leq j \leq n$, \mbox{$x = \diag(X)$} and $\nu$; 
\STATE $\nu \leftarrow 2 \nu$; 
\ENDWHILE 
\STATE $\tilde{W} = \tilde{M}(:,\mathcal{K})$ ; 
\end{algorithmic} 
\end{algorithm}  
\algsetup{indent=2em}
\begin{algorithm}[ht!]
\caption{Cluster Extraction \label{cluster}}
\begin{algorithmic}[1] 
\REQUIRE A set of points $m_j$ $1 \leq j \leq n$, a vector of weights $x \in \mathbb{R}_+^n$ such that \mbox{$\sum_{i=j}^n x = r$}, and $\nu \geq 0$. 
\ENSURE A index set $\mathcal{K}$ of centroids corresponding to clusters with weight strictly larger than $\frac{r}{r+1}$. \medskip 

\STATE  $D(i,j) = ||m_i-m_j||_1$ for $1 \leq i, j \leq n$. 
\STATE  $\mathcal{S}_i = \{ j \ | \ D(i,j) \leq \nu \}$ for $1 \leq i \leq n$; 
\STATE  $w(i) = \sum_{j \in \mathcal{S}_i} x(j)$ for $1 \leq i \leq n$; 
\STATE $\mathcal{K} = \emptyset$; 
\WHILE{$\max_{1 \leq i \leq n} w(i) > \frac{r}{r+1}$} 
\STATE $k = \argmax w(i)$; 
\STATE $\mathcal{K} \leftarrow \mathcal{K} \cup \{ k \}$; 
\STATE For all $j \in \mathcal{S}_k$ : $w(j) \leftarrow 0$; 
\STATE For all $i \notin \mathcal{S}_k$ and $j \in \mathcal{S}_k$ such that $j \in \mathcal{S}_i$ : $w(i) \leftarrow w(i)-x(j)$; 
\ENDWHILE 
\end{algorithmic} 
\end{algorithm}

\begin{lemma} \label{lem6}
Let $m_j \in \mathbb{R}^m$ $1 \leq j \leq n$, $x \in \mathbb{R}_+^n$ be such that $\sum_{j=1}^n x = r$, and $\rho \geq 0$. Let also 
%the $r$ clusters $\Omega_k$ be 
%defined $r$ clusters 
$\Omega_k = \{ m_j \ | \ ||m_j - w_k||_1 \leq \rho \} \text{ for } 1 \leq k \leq r$ 
where $w_k \in \mathbb{R}^m$ $1 \leq k \leq r$. 
Suppose 
\begin{itemize}
\item $\sum_{j \in \Omega_k} x_j > \frac{r}{r+1}$, 

\item $ \omega = \min_{i \neq j} ||w_i - w_j||_1 > 6 \rho$, and 

\item For all $1 \leq k \leq r$, there exists $1 \leq j \leq n$ such that $||m_j - w_k||_1 \leq \epsilon \leq \rho$. 
\end{itemize}
Then, for any $(\rho + \epsilon) \leq \nu \leq 2(\rho + \epsilon)$, 
Algorithm~\ref{cluster} identifies a set $\mathcal{K}$ with $r$ indices such that 
\begin{equation} \label{bnlem6}
\max_{1 \leq k \leq r} \min_{j \in \mathcal{K}} ||m_j - w_k||_1 \leq 3 \rho + 2 \epsilon. 
\end{equation}
Moreover, if Algorithm~\ref{cluster} identifies a set $\mathcal{K}$ with $r$ indices for some $\nu < \rho + \epsilon$, then $\mathcal{K}$ satisfies \eqref{bnlem6}. 
\end{lemma}
\begin{proof} 
First notice that the index set $\mathcal{K}$ extracted by Algorithm~\ref{cluster} cannot contain more than $r$ indices. In fact, Algorithm~\ref{cluster} only identifies clusters with weight strictly larger than $\frac{r}{r+1}$ while the total weight $\sum_{i=1}^n x$ is equal to $r$. It remains to show that $\mathcal{K}$ contains at least $r$ indices. 

Let first consider the case $(\rho + \epsilon) \leq \nu \leq 2(\rho + \epsilon)$. Let $\mathcal{S}_i$ $1 \leq i \leq n$ be the sets computed by Algorithm~\ref{cluster} before entering the while loop.  We observe that 
\begin{itemize} 

\item For $m_j \in \Omega_k$ and $m_{j'} \in \Omega_{k'}$ where $j \neq j'$ and $k \neq k'$, we have $m_j \notin \mathcal{S}_{j'}$ and $m_{j'} \notin \mathcal{S}_{j}$. In fact, 
\[
||m_j - m_{j'}||_1 = ||(m_i- w_k) + (w_k - w_{k'}) + (w_{k'} - m_{j'})||_1 \geq \omega - 2 \rho > 4 \rho \geq \nu. 
\]

\item For all $1 \leq k \leq r$, there exists $m_j \in \Omega_k$ such that $w(j) > \frac{r}{r+1}$. By assumption, for all $1 \leq k \leq r$, there exists $m_j \in \Omega_k$ such that $||m_j - w_k||_1 \leq \epsilon$, hence for all $m_i \in \Omega_k$ we have  $||m_j - m_i||_1 = ||(m_j - w_k) + (w_k - m_i)||_1 \leq \rho + \epsilon \leq \nu$ while $\sum_{i \in \Omega_k} x(i) > \frac{r}{r+1}$. 

\item If $m_i \notin \cup_{1 \leq k \leq r} \Omega_k$ and $w(i) > \frac{r}{r+1}$, then $||m_i-w_k||_1 \leq 3 \rho + 2 \epsilon$ for some $1 \leq k \leq r$. Suppose $||m_i-w_k||_1 > 3 \rho + 2 \epsilon$ for all $k$, then for all $m_j \in \cup_{1 \leq k \leq r} \Omega_k$ 
\[
||m_i - m_j ||_1 \geq ||(m_i - w_k) + (w_k - m_j) ||_1 > 3 \rho + 2 \epsilon - \rho \geq \nu.
\]
Therefore, $\sum_{j \in \mathcal{S}_i} x(j) \leq r - \sum_{k} \sum_{j \in \Omega_k} x_i < r - r \frac{r}{r+1} < \frac{r}{r+1}$, a contradiction. 

Let then $k$ be such that $||m_i-w_k||_1 \leq 3 \rho + 2 \epsilon$. This implies that if $m_j \in \mathcal{S}_i$, then either $m_j \in \Omega_k$, or $m_j \notin  \cup_{k' \neq k} \Omega_{k'}$. In fact, if $m_j \in \Omega_{k'}$ for some $k' \neq k$, then
\begin{align*}
||m_i - m_j||_1 & \geq ||(m_i - w_k) + ( w_k-w_{k'}) + (w_{k'} - m_j)||_1 \\
& \geq \omega - 3\rho - 2\epsilon - \rho > 2 \rho - 2\epsilon \geq \nu, 
\end{align*}
a contradiction. 
\end{itemize}
These observations imply that there are at least $r$ disjoint sets $\mathcal{S}_i$ with weight larger than $\frac{r}{r+1}$, each corresponding to a different cluster $\Omega_k$. Therefore, Algorithm~\ref{cluster} will identify them individually and \eqref{bnlem6} will be satisfied.  

For the case $\nu < \rho + \epsilon$, the result follows directly from the observations above: any point $m_i$ with $w(i) > \frac{r}{r+1}$ must satisfy $||m_i - w_k||_1 \leq 3\rho + 2 \epsilon$ for some $1 \leq k \leq r$. 
Moreover, for all $k$ there must exist $j \in \mathcal{K}$ such that $||m_j - w_k||_1 \leq 3\rho + 2 \epsilon$. In fact, suppose there exists $k$ such that $||m_j - w_k|| > 3\rho + 2 \epsilon$ for all $j \in \mathcal{K}$. Then, $m_i \notin \cup_{j \in \mathcal{K}}\mathcal{S}_j$ for all $i \in \Omega_k$  (see above) hence 
\[
\sum_{i \in \mathcal{S}_j, j \in \mathcal{K}} x(i) < r - \frac{r}{r+1} = r \frac{r}{r+1}, 
\] 
which implies that $\mathcal{K}$ cannot contain more than $r-1$ indices, a contradiction. \vspace{0.2cm} 
\end{proof}

 \begin{theorem} \label{mainTh} 
  Let $M=WH$  be a normalized $r$-separable matrix with $W$ $\kappa$-robustly conical. 
  Let also $\tilde{M} = M+N$  with $||N||_{1} \leq \epsilon$. If 
   \begin{equation} \label{ourbo}
  \epsilon < \frac{\omega \kappa}{99 (r+1)}, 
  %\quad \text{ where } \quad \delta \leq \min_{i \neq j} ||W(:,i)-W(:,j)||_1, 
  \end{equation}  
  where $\omega = \min_{i \neq j} ||W(:,i)-W(:,j)||_1$,   then Algorithm~\ref{postpro} will extract a matrix $\tilde{W}$ such that 
  \[
  ||W - \tilde{W}(:,P)||_{1} \leq \delta =  49(r+1) \frac{\epsilon}{\kappa} + 2 \epsilon, \quad \text{   for some permutation $P$.} 
  \]
  \end{theorem}
\begin{proof} 
%To have the clusters $\Omega_k^{\delta}$ well-separated, it suffices for $\delta$ to be strictly smaller than half the distance between each pair cluster centroids, that is, 
%\[
%  \delta < {\omega} =  \frac{1}{2} \min_{i \neq j} ||W(:,i)-W(:,j)||_1 \leq 1. 
%  \] 
%If $\epsilon = 0$, the proof is given in \cite[Prop.\@ 3.1]{BRRT12}. Otherwise $\epsilon > 0$. 
Let $X$ be a feasible solution of \eqref{rechtLP}, let the $r$ clusters  $\Omega_k^{\rho}$ $1 \leq k \leq r$ be defined as in Equation~\eqref{clusters} and let $c_k = \sum_{j \in \Omega_k^{\rho}} X(j,j)$. 
%  and the proof is complete. 
%Let us focus on the sufficient condition $\epsilon \leq \frac{\delta \kappa}{10(r+1)}$. 
If $\rho < \frac{\omega}{6}$ and $c_k > \frac{r}{r+1}$, then, by Lemma~\ref{lem6}, Algorithm~\ref{cluster} will identify a set $\mathcal{K}$ with $r$ indices such that % If $\epsilon = 0$, the proof is given in \cite[Prop.\@ 3.1]{BRRT12}. Otherwise $\epsilon > 0$. 
\[
\max_{1 \leq k \leq r} \min_{j \in \mathcal{K}} ||W(:,k) - \tilde{M}(:,j)||_1 \leq \delta = 3\rho + 2 \epsilon, 
\]
for any $\nu \in [\rho + \epsilon, 2 \rho + 2 \epsilon]$. 
Therefore, starting with $\nu = 2\epsilon \leq (\rho + \epsilon)$ and multiplying it by two at each iteration will eventually give a value of $\nu$ in $[\rho + \epsilon, 2 \rho + 2 \epsilon]$. (Note that Algorithm~\ref{cluster} could return a set $\mathcal{K}$ with $r$ indices for $\nu$ smaller than $\rho + \epsilon$, see Lemma~\ref{lem6}. Note also that the number of iterations performed by Algorithm~\ref{postpro} is at most $\log_2\left( \frac{\rho + \epsilon}{\epsilon} \right)$.) If $\epsilon = 0$, then $c_k =1$ for all $1 \leq k \leq r$ while $\rho = 0 < \frac{\omega}{6}$, and the loop is entered at most once (if the entries of $p$ are distinct, then it is not entered because exactly $r$ diagonal entries of an optimal solution of \eqref{rechtLP} will be equal to one, each corresponding to a different column of $W$ \cite[Prop.~3.1]{BRRT12}). 
Otherwise $\epsilon > 0$ and it remains to guarantee that $\rho < \frac{\omega}{6}$ and $c_k > \frac{r}{r+1}$. 
By Lemma~\ref{lem5},  %for $c_k > \frac{r}{r+1}$ for all $1 \leq k \leq $, it suffic
    \[
  \frac{\epsilon}{1-\epsilon}  < \frac{\rho \kappa}{16 (r+1)} 
  \quad 
  \Rightarrow 
  \quad 
  c_k > \frac{r}{r+1}. 
  \] 
 Taking $\epsilon < \frac{\omega \kappa}{99(r+1)}$ and $\rho =  \frac{98}{6}\, (r+1) \frac{\epsilon}{\kappa} < \frac{\omega}{6}$ completes the proof since 
%, while we can take any 
\[
\rho =  \frac{98}{6} (r+1) \frac{\epsilon}{\kappa} >  16 (r+1) \frac{\epsilon}{\kappa}  \left( \frac{1}{1-\epsilon} \right)  , 
\]  
%gives the result. 
%since 
because $\frac{96}{1-\epsilon} < 98$ for any $0 \leq \epsilon < \frac{1}{49}$.  \vspace{0.2cm}  
%(as $99(r+1) > 100$ for $r\geq 1$). 
\end{proof}

It can be checked that all the results from Section~\ref{ip33} apply to Algorithm~\ref{postpro}. 
In fact, by assumption, all the matrices considered did not contain duplicate nor near-duplicate of the columns of matrix $W$ in which case we showed that $r$ diagonal entries of $X$ have weight at least $\frac{r}{r+1}$. This implies that Algorithm~\ref{postpro} will not enter the while loop, hence it is equivalent to Algorithm~\ref{balgo2}.  
In particular, Corollary~\ref{cor2} also applies to Algorithm~\ref{postpro}, that is, it is necessary that   
\[
\epsilon < \frac{\kappa}{r-1}, \qquad \text{ for any $\delta < \frac{\kappa}{2}$. }
\]
This shows that the bound of Theorem~\ref{mainTh} for $\epsilon$ is tight up to a factor $\omega$ (and some constant multiplicative factor).  
Moreover, by Theorem~\ref{lbdelta}, the bound for $\delta$ is tight up to a factor $r$ (and some constant multiplicative factor).

\begin{remark}[Computational Cost] \label{ccost}
The main additional cost of Algorithm~\ref{postpro} compared to Algorithm~\ref{balgo2} is to computing and storing  the distance matrix $D$. This requires $\mathcal{O}(mn^2)$ floating point operations and $\mathcal{O}(n^2)$ space in memory. 
This is negligible as computing $MX$ already requires $\mathcal{O}(mn^2)$ operations, while storing $X$ requires $\mathcal{O}(n^2)$ space in memory. 
%Notice that if $\diag(X)$ contains zero entries, they can be discarded along with the corresponding data points. 
\end{remark}

\begin{remark}[Choice of the vector $p$] Because of the post-processing procedure in Algorithm~\ref{postpro}, it is not necessary for Theorem~\ref{mainTh} to hold that the vector $p$ has distinct entries. 
However, it will still be useful in practice to impose this condition. In fact, this will incite the weights to be concentrated in fewer diagonal entries of $X$ so that typically fewer loops will have to be performed to obtain a set $\mathcal{K}$ containing $r$ indices. 
In particular, in the exact case (that is, $\epsilon = 0$) or in the case there is no duplicate and near duplicate in the dataset (see above), the loop will not be entered. 
%avoid some duplicates to have 
% to a good practice because then, 
%in case of repetition in the dataset, not many loops will have to be performed; in particular in the exact case. 
\end{remark}

\begin{remark}[More Sophisticated Post-processing Strategies] 
It is possible to design better post-processing procedures but we wanted here to keep the analysis simple. 
%For example, when a index $k$ is added to the set $\mathcal{K}$, it would make sense to subtract the weight of all the points belonging to $S_k$ from the other sets $S_i$'s they belong to (so that their weight is not accounted twice). 
In particular, if the input matrix $\tilde{M}$ does not satisfy the conditions of Theorem~\ref{mainTh}, it may happen that no set $\mathcal{K}$ computed in the loop of Algorithm~\ref{postpro} contains $r$ elements. 
Therefore, one should keep in memory the largest set extracted so far, or  design more sophisticated strategies. 
For example, if less than $r$ clusters have been extracted, the condition that the weight of each extracted cluster must larger than $\frac{r}{r+1}$ can be relaxed; this variant has been implemented in the Matlab code available at \url{https://sites.google.com/site/nicolasgillis/code}. 
%a cluster to be extracted that its weight is larger than some constant smaller than $\frac{r}{r+1}$
%only require for a cluster to be extracted that its weight is larger than some constant smaller than $\frac{r}{r+1}$ .  
%are possible and probably the one proposed in Algorithm~\ref
\end{remark}

\begin{remark}[Pre-processing] 
Another possible way to deal with duplicates and near duplicates would be to use an appropriate pre-processing. 
In \cite{EMO12}, $k$-means is used to reduce the number of data points and get rid of the duplicates. In fact, their algorithm cannot deal with duplicates, even in the noiseless case (note that their robustness result is only asymptotical, that is, it only holds when th noise level $\epsilon$ goes to zero). 
Arora et al.~\cite{AGKM11} also use some pre-processing in their algorithm (before processing any data point, its neighbors have to be discarded). However, it seems difficult to combine a robustness analysis with a pre-processing strategy (in fact, Arora et al.~\cite{AGKM11} need the conditioning $\alpha$ as an input to do so); this is a topic for further research. 
%However, they need to know the parameter $\alpha$, hence it seems difficult to apply their pre-processing as there does not seem to be an easy way to evaluate $\alpha$ given only the data matrix and the noise level. 
\end{remark}

\subsection{Repartition of the Weights inside a Cluster} \label{repar}

In this section, we show that Hottopixx (Algorithm~\ref{balgo2}) cannot provide better bounds than Algorithm~\ref{postpro}. The reason is the following: inside a cluster $\Omega_k^{\rho}$,  there is no guarantee that all the weight will be assigned to a single diagonal entry of $X$. %In fact, %because $\delta \geq \frac{\epsilon}{\alpha} + \frac{3}{2} \epsilon \geq 2 \epsilon$ (see Theorem~\ref{lbdelta}), the weight might have to be assigned to more than one entry. 
In the proof of Theorem~\ref{th7}, we show that the weight may be equally distributed inside a cluster. 
This construction allows us to show that $\epsilon \leq \frac{\kappa}{(r-1)^2}$ is  necessary for Proportion~\ref{prop2} to hold for any $\delta < \kappa + \epsilon$, which proves our claim. 

\begin{theorem} \label{th7}
For any $r \geq 3$ and $\delta < \kappa + \epsilon$, it is necessary for Proposition~\ref{prop2} to hold that 
\[
\epsilon \leq \frac{\kappa}{(r-1)^2}. 
\]
\end{theorem}
\begin{proof}
See Appendix~\ref{appc}. \vspace{0.2cm}
\end{proof}

%To conclude this section, 
\begin{remark}
%We would like to point out that 
%although Hottopixx cannot in general tolerate a larger noise level than Algorithm~\ref{postpro}, 
It remains an open question whether there exists a bound on the noise level to guarantee Hottopixx to be robust for any separable matrix, that is, one that also contains duplicates and near duplicates (note that, by Theorem~\ref{th7}, this bound, if it exists, has to be smaller than $\frac{\kappa}{(r-1)^2}$). 
\end{remark}

\subsection{Comparison with the Algorithm of Arora et al.~\cite{AGKM11}} \label{aror} 

In this section, we compare the theoretical bounds for Algorithm~\ref{postpro} obtained in Theorem~\ref{mainTh} with the ones of the algorithm of Arora et al.\@ for which the following holds.  
\begin{theorem}[\cite{AGKM11}, Th.\@ 5.7] \label{tharor}
Let $M=WH$  be a normalized $r$-separable matrix with $W$ $\alpha$-robustly simplicial. 
  Let also $\tilde{M} = M+N$  with $||N||_{1} \leq \epsilon$. If 
   \begin{equation} \label{arobo}
  \epsilon < \frac{\alpha^2}{20+13 \alpha}, 
  %\quad \text{ where } \quad \delta \leq \min_{i \neq j} ||W(:,i)-W(:,j)||_1, 
  \end{equation} 
 then the algorithm proposed by Arora et al.~\cite{AGKM11}  extracts a matrix $\tilde{W}$ such that 
  \[
  ||W - \tilde{W}(:,P)||_{1} \leq 10\frac{\epsilon}{\alpha} + 6 \epsilon ,  \quad \text{   for some permutation $P$.} 
  \] 
\end{theorem}

There are two bounds to compare. First, there is the bound on noise level $\epsilon$ allowed to have any error guarantee. 
%\begin{enumerate}
%\item T (hence the higher, the better), 
%\item The bound on the error of the solution obtained, given that the noise level is sufficiently small (hence the smaller, the better).  
%\end{enumerate} 
%The bound on the noise level $\epsilon$ 
The one from Equation~\eqref{arobo} does not dominate the one from Theorem~\ref{mainTh}, see Equation~\eqref{ourbo}. 
In fact, $\alpha$ and $\kappa$ only differ by a factor of at most two (Theorem~\ref{kappaalpha}) while $\omega$ ($\geq \kappa \geq \frac{\alpha}{2}$) can potentially be arbitrarily larger than $\alpha$ (take for example the columns of $W$ as the vertices of a flat triangle). 
Hence, for some highly ill-conditioned matrices, Algorithm~\ref{postpro} can tolerate much higher noise levels.  
 %(while we will show in Section~\ref{ne} that Algorithm~\ref{postpro} performs better than Hottopixx). 

Second, there is the bound on the error: the algorithm of Arora et al.\@ dominates the one of Algorithm~\ref{postpro}, but only up to a factor $r$ (which is usually small in practice). 
This is not very surprising since the algorithm of Arora et al.\@ is optimal in terms of the error bound; see Section~\ref{pnc}. 
However, \emph{the algorithm of Arora et al.\@  requires the parameter $\alpha$ as an input,} which, we believe, is highly impractical. At least, we do not know of an efficient way to compute $\alpha$ (and this issue is not discussed in their paper).  Moreover, it was observed in \cite{BRRT12} that Hottopixx performs better than the algorithm of Arora et al.\@ on some synthetic datasets.

To conclude, Algorithm~\ref{postpro} is, to the best of our knowledge, the provably most robust algorithm for separable NMF for which the condition number $\alpha$ is not required as an input.

\section{Numerical Experiments} \label{ne}

  In this section, we present some numerical experiments to show the superiority of Algorithm~\ref{postpro} over Hottopixx in case there are duplicates and near duplicates of the columns of $W$ in the data set, otherwise both algorithms coincide since the post-processing will not be entered (cf.\@ the discussion after Theorem~\ref{mainTh}). 
  %In other words, we show that using the post-processing strategy allows one to improve the performances of Hottopixx. 
  All experiments were run on a two-core machine with 2.99GHz and 2GB of RAM using a CPLEX implementation to solve the LP \eqref{rechtLP}; the code was developed in \cite{GL13}, and is available at 
%\begin{center}
  \url{https://sites.google.com/site/nicolasgillis/code}. 
%\end{center}
%The constructions of Theorems~\ref{th7} and the post-processing procedure (Algorithm~\ref{postpro}) can be found on the same web page. 
 %\subsection{Construction from Theorem~\ref{th7}}  
%In this section, 
We use the constructions from the proof of Theorem~\ref{th7} with the following parameters: $\kappa = 0.1$ and $K = 5$, while we vary the value of the rank $r$ and the noise level $\epsilon$.  
Moreover, we duplicate each column of $W$ twice (that is, each column of $W$ is present three times in the data set), and permute the columns of $\tilde{M}$ at random in order to avoid a bias towards the natural ordering. Finally, we slightly perturb the vector $p$ in the objective function to make  its entries distinct (since we also duplicated the entries of $p$) by adding to each entry a value drawn from the normal distribution with mean zero and standard deviation $0.1$.  

\begin{figure*}[!h]
\begin{center}
    \includegraphics[width=.49\textwidth]{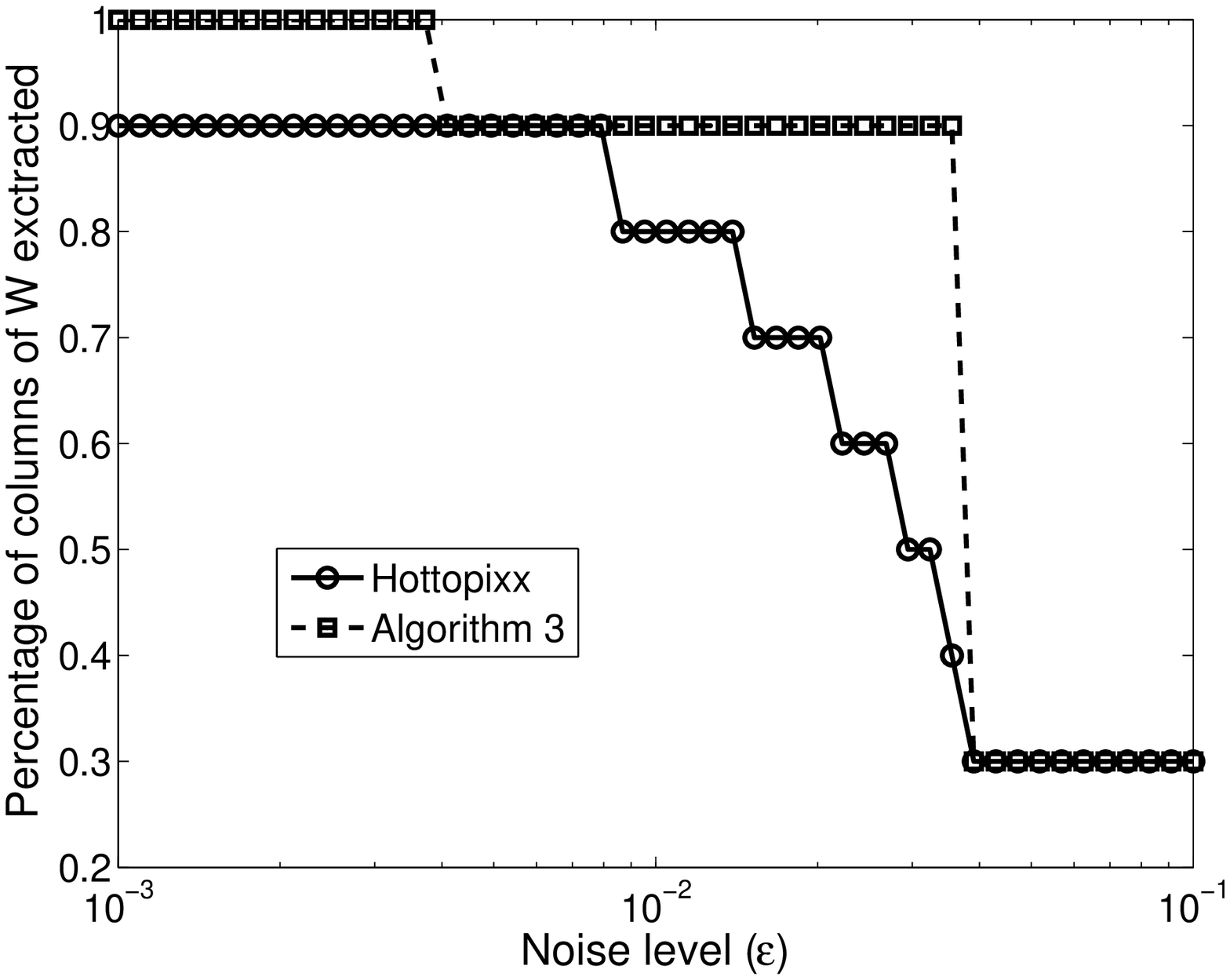}
    \hfill
    \includegraphics[width=.49\textwidth]{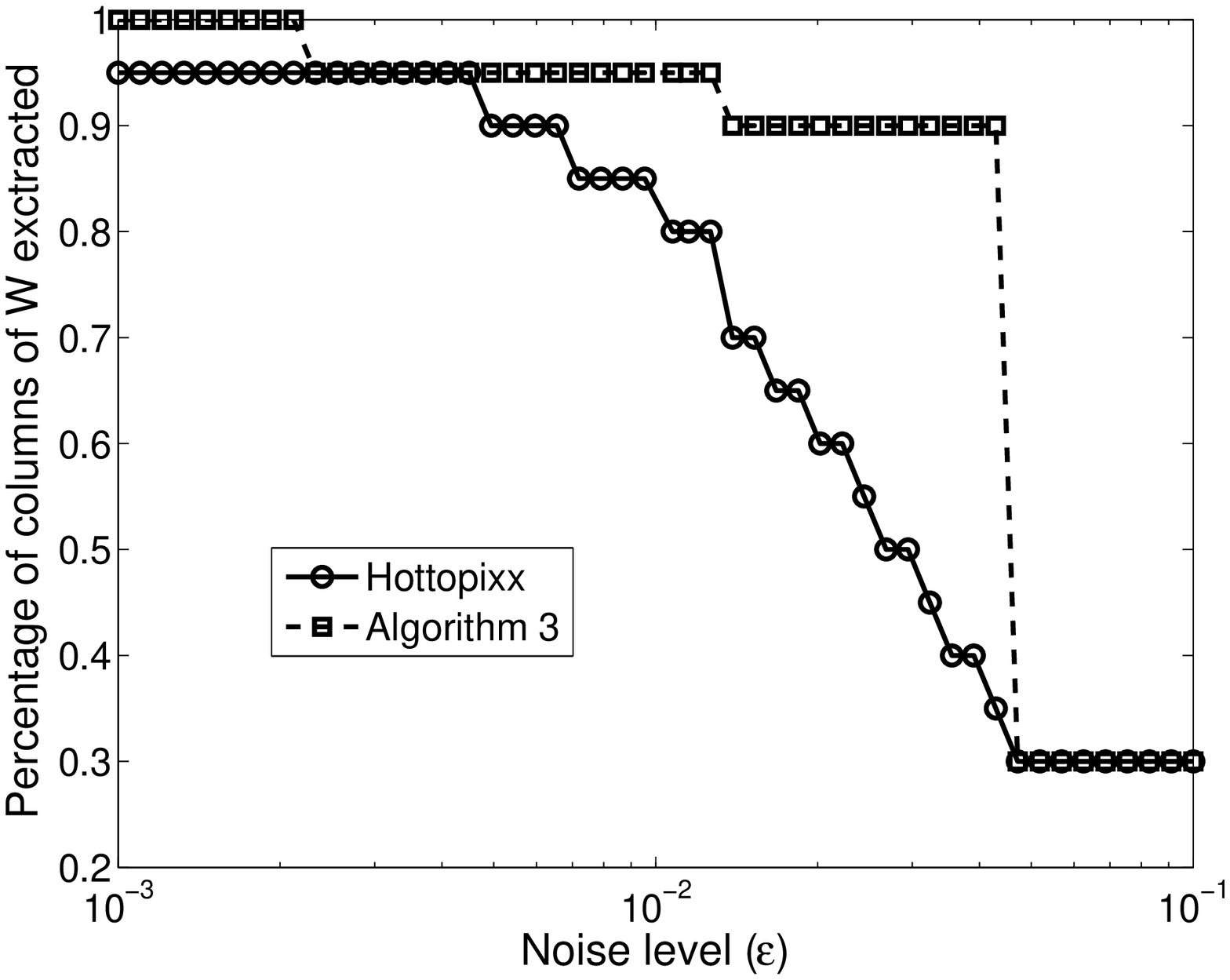} \\
    \includegraphics[width=.49\textwidth]{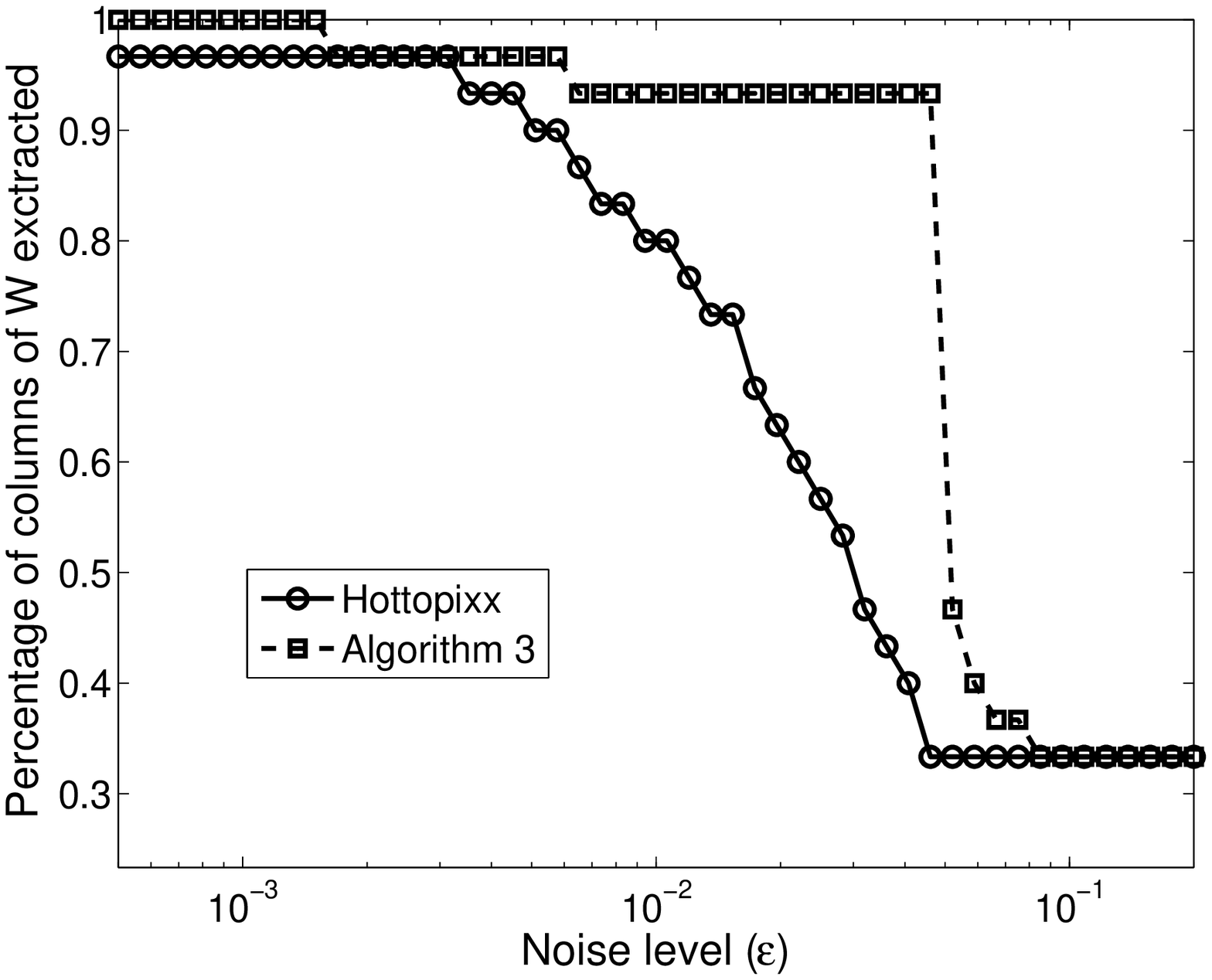}
    \hfill
    \includegraphics[width=.49\textwidth]{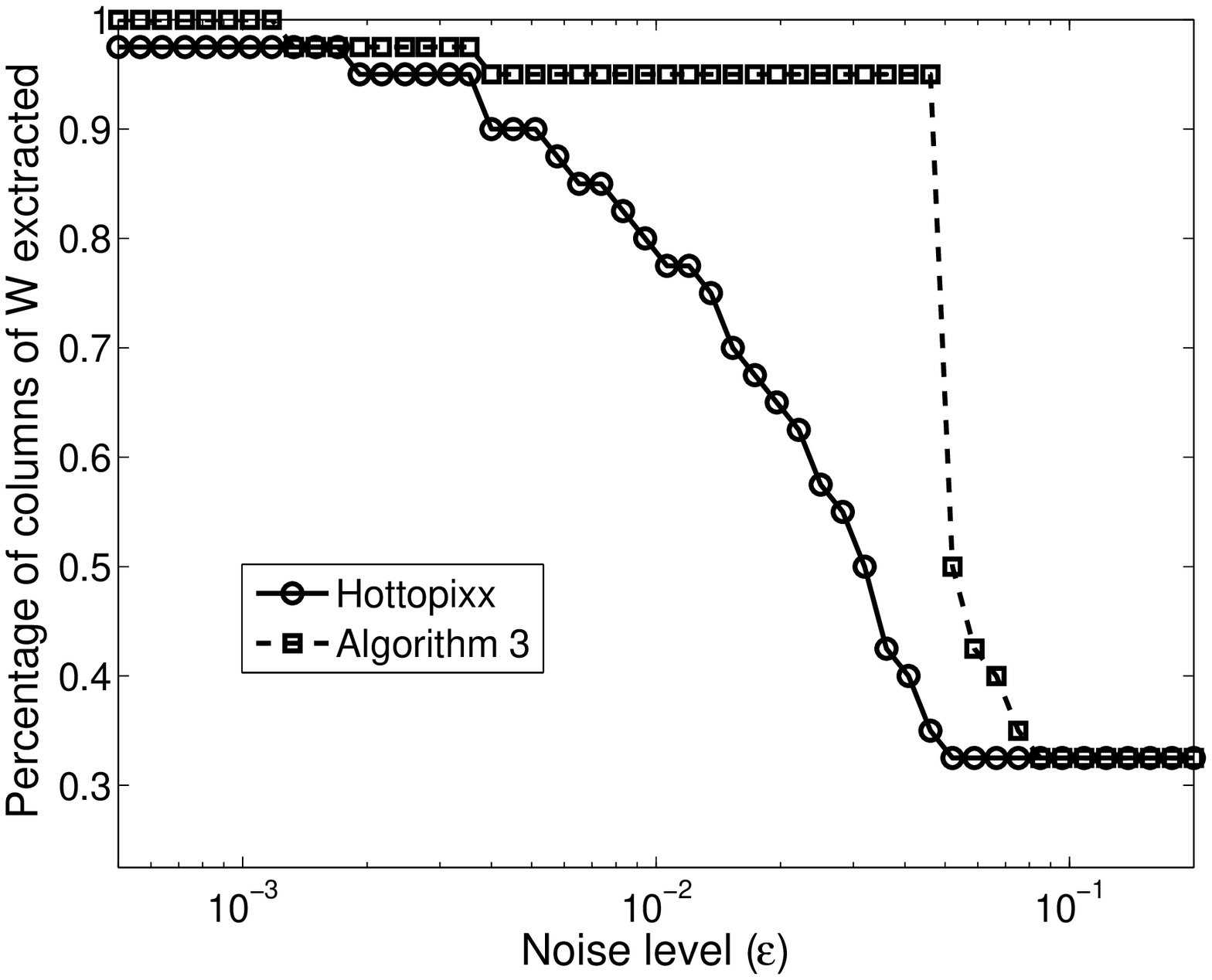}
    \caption{Comparison of Hottopixx and Algorithm~\ref{postpro} on the near-separable matrices from Theorem~\ref{th7} with duplicates of the columns of $W$. From left to right, top to bottom: $r = 10, 20, 30, 40$.}
    \label{th36}
\end{center}
\end{figure*} 
%Since the vector $p$ in the objective function
Figure~\ref{th36} displays the percentage of columns of $W$ correctly extracted by the two algorithms for different values of the rank $r$ and of the noise level $\epsilon$ (hence the higher the curve, the better).  
  As shown in Theorem~\ref{th7}, Hottopixx cannot extract one of the columns of $W$, even for small noise   levels. 
  More interestingly, Algorithm~\ref{postpro} clearly outperforms Hottopixx for larger noise levels. 
  For example, for $r = 40$ and $\epsilon = 0.046$ (see the plot in the bottom right corner of  Figure~\ref{th36}), Hottopixx only identifies 35\% of the forty columns of $W$ while Algorithm~\ref{postpro} identifies 95\% of them.   
  The reason for this behavior is the following: when the noise is large, the set of feasible solutions of \eqref{rechtLP} is typically larger (because the constraint $||M-MX||_1 \leq 2\epsilon$ is relaxed). For $\epsilon$ sufficiently large, this allows all duplicates of some columns of $W$ corresponding to small entries of $p$ to be given a large weight, hence some columns of $W$ are extracted more than once. This is not possible with the post-processing which clusters these entries together, and is then able to extract other columns of $W$ whose corresponding diagonal entries of $X$ are smaller. 
  
 Note that the computational time of the post-processing strategy is negligible and takes in average about 5\%  percent of the total time needed to solve the LP~\eqref{rechtLP}.

  \section{Conclusion and Further Work}

 In this paper, we have proposed a provably more robust variant of Hottopixx based on an appropriate post-processing of the solution of the linear program \eqref{rechtLP} (see Algorithm~\ref{postpro}). 
 In particular, we proved that Algorithm~\ref{postpro} is robust for any input separable matrix $M$ (Theorem~\ref{mainTh}), while our analysis is close to being tight. 
 
 %deals with near duplicates in the dataset. Hence our result can be applied to any noisy separable matrix.  This allowed us to design a more robust variant using an appropriate post-processing of the solution of the linear program \eqref{rechtLP} (see Algorithm~\ref{postpro}) which is provably more robust than the original Hottopixx algorithm (see Algorithm~\ref{balgo2}). 
 %It would be interesting to try to tighten the bound from Theorem~\ref{Th2} to find out whether 
 
 It would be interesting to improve the bound of Theorem~\ref{mainTh} or show that the bound is tight. It would also be particularly interesting to design more robust or computationally more effective (or both?) separable NMF algorithms. In particular, the following question seems to be open:  
 does it exist a polynomial-time algorithm to which only the noisy separable matrix $\tilde{M}$ and the noise level $\epsilon$ are given as input, and that achieves an error of order $\mathcal{O}\left( \frac{\epsilon}{\alpha} \right)$? Such an algorithm would be optimal (see Section~\ref{pnc}). 
 Note that the algorithm of Arora et al.~\cite{AGKM11} achieves this bound but requires the parameter $\alpha$ as an input, which is highly impractical (there does not seem to be an easy way to evaluate $\alpha$). %Moreover, Algorithm~\ref{postpro} is optimal up to a factor $r$. 

 \section*{Acknowledgments}
 
The author would like to thank Robert Luce (T.U.\@ Berlin) for some insightful discussions and for pointing out an error in Lemma~\ref{lem2} in a previous draft. The author would also like to thank the reviewers for their feedback which helped improve the paper significantly. \vspace{0.2cm}

\appendix

\section{Proof of Theorem~\ref{kappaalpha}} \label{appa} 

\begin{proof}[\textbf{Proof of Theorem~\ref{kappaalpha}}]
Let  
\[
k = \argmin_{1 \leq j \leq r} \min_{x \in \mathbb{R}^{r-1}_+} ||W(:,j) - W(:,\mathcal{J})x||_1, \quad \text{ where } \mathcal{J} = \{1,2,\dots,r\} \backslash \{j\},
\] 
$w = W(:,k)$, and y$ = W(:,\mathcal{R})x^*$ where
 \[
x^* = \argmin_{x \in \mathbb{R}^{r-1}_+} ||W(:,k) - W(:,\mathcal{R})x||_1, \quad \text{ where } \mathcal{R} = \{1,2,\dots,r\} \backslash \{k\}, 
\] 
so that, by definition, $||y-w||_1 = \kappa$. 
If $y = 0$, we are done since $\kappa = ||w||_1 = 1 \geq \frac{1}{2} \alpha$ as $\alpha \leq 2$.
Otherwise $y \neq 0$ and we define $z = \frac{y}{||y||_1} = \lambda^{-1} y$. By definition, $||w-z||_1 \geq \alpha$ since $z$ belongs to the convex hull of the columns of $W$. 
We have 
\begin{align*}
\alpha \leq ||w-z||_1 & = ||w-(\lambda + 1-\lambda )z||_1 
 \leq ||w-\lambda z||_1 + (1-\lambda ) ||z||_1 
 = \kappa + (1-\lambda ) \leq 2 \kappa 
\end{align*}
since $\kappa = ||w-\lambda z||_1 \geq ||w||_1-||\lambda z||_1 = 1-\lambda$, and the proof is complete. 
\end{proof}

\vspace{0.5cm}

\section{Proof of Theorem~\ref{nc1}}  \label{appb}

The following lemma shows  that if one of the coefficients in the objective function of a linear program is much larger than all the other ones, then the corresponding entry of any optimal solution must be smaller than the corresponding entry of any feasible solution. Although the result is clear intuitively, we provide here a simple proof. 
%Intuitively, the level set get parallel to the
\begin{lemma} \label{lplem} 
Let consider the following linear program 
\begin{equation} \label{lp}
\min_{x \in \mathbb{R}^n} c_K^T x \quad \text{ such that } \; Ax = b \; \text{ and } \; l \leq x \leq u, 
\end{equation}
with $l, u \in \mathbb{R}^n$, $l \leq u$, and $c_K = (K,\tilde{c}) \in \mathbb{R}^n$ where $K \in \mathbb{R}$ is a parameter. Let us denote $x^*_K$ an optimal solution of \eqref{lp} depending on $K$. 
Assume there exists a feasible solution $x^f$ of \eqref{lp} such that $x^f(1) = s$. Then, for any $K$ sufficiently large, $x^*_K(1) \leq s$. 

Similarly, if $c_K(1) = -K$ and there exists a feasible solution such that $x(1) = t$, Then, for any $K$ sufficiently large, $x^*_K(1) \geq t$.  
\end{lemma}
\begin{proof}
Let $\mathcal{V} %= \{ x \ | \ Ax = b \text{ and } l \leq x \leq u \} 
\neq \emptyset$ be the set of vertices of the feasible set of \eqref{lp}, and $\bar{\mathcal{V}} = \{ x \in \mathcal{V} \ | \ x(1) > s \}$. Notice that because the feasible set of \eqref{lp} is a polytope, there always exists an optimal solution in $\mathcal{V}$. Let us denote $d = \min_{x \in \bar{\mathcal{V}}} x(1) > s$. Assume there exists an optimal solution $x^*_K$ such that $x^*_K(1) > s$. 
%Since the feasible region of \eqref{lp} is bounded, we can assume 
%Without loss of generality, we can assume that $x^*_K \in \bar{\mathcal{V}}$. 
This implies that there exists an optimal solution $\bar{x}^*_K \in \bar{\mathcal{V}}$ (since any optimal solution is a convex combination of optimal vertices in ${\mathcal{V}}$). 
Therefore, 
\[
K d - ||\tilde{c}||_2 ||u||_2 
\leq c_K^T x^*_K = c_K^T \bar{x}^*_K 
= K \bar{x}^*_K(1) + \tilde{c}^T \bar{x}^*_K(2\text{:}n)  
\leq c_K^T x^f 
%= Ks  + \tilde{c}(2\text{:}n)^T x^f_K(2\text{:}n) 
\leq Ks  + ||\tilde{c}||_2 ||u||_2 , 
\]
which is absurd for any $K > \frac{2||\tilde{c}||_2 ||u||_2}{d-s}$. 
\end{proof}

The linear program \eqref{rechtLP} can be written in the form of \eqref{lp}; in fact, $0 \leq X \leq 1$ while the $mn$ additional variables necessary to express the constraint $||M-MX||_{1} \leq 2\epsilon$ linearly will be in the interval $[0,2\epsilon]$. Therefore, Lemma~\ref{lplem} applies to \eqref{rechtLP}.

\begin{proof}[\textbf{Proof of Theorem~\ref{nc1}}]
We prove the result with the following construction: Let 
\[
W = \left(
\begin{array}{c}
\frac{\kappa}{2} I_{r}  \\
(1-\frac{\kappa}{2})  e^T % \\ 
%0_{1\times r}  
\end{array} 
\right), 
\] 
which is $\kappa$-robustly conical. 
Let also 
\[
H = \left( 
\begin{array}{cc}
I_r & \beta I_r + (E_{r} - I_r) \frac{1-\beta}{r-1} 
\end{array} 
\right), 
\] 
so that $\max_{i,j} H'_{ij} = \beta$ (note that $\beta$ must be larger than $\frac{1}{r}$ since the columns of $H'$ sum to one), $N = 0$,  
$\tilde{M} = WH + N$, $p = (1,2,3,\dots, r-1, -K, -1, -2,\dots,-(r-1),-K^2)^T$ for $K$ sufficiently large, and 
\[ 
 \epsilon =  \frac{\kappa (1-\beta)}{(r-1)(1-\beta) + 1} \leq \frac{\kappa}{r-1} . 
 % \leq \frac{\kappa (1-\frac{1}{r})}{r-1} = \frac{\kappa}{r}.  
\]
Assume that 
\[
X = \left( \begin{array}{cccc} 
(1-\delta) I_{r-1} + \frac{\delta-\omega}{r-1}  J_{r-1}  & 0
& (1-\delta) \left( \beta I_{r-1} + \frac{1}{r-1} J_{r-1} \left( \frac{1-\omega}{1-\delta} - \beta  \right)\right) & 0  \\ 
\frac{\delta-\omega}{r-1} e^T & 1 & \frac{1-\delta}{r-1} \left( \frac{1-\omega}{1-\delta} - \beta \right) e^T & 0 \\  
\omega I_{r-1}  & 0 & \omega I_{r-1}   & 0  \\ 
0 & 0 &  0 &  1 \\  
\end{array} \right)
\] 
where $J_{r-1} = E_{r-1} - I_{r-1}$,  %\bar{r} = r-1
\[
\delta = (2-\beta) \omega  
\quad \text{and} \quad 
\omega = \frac{\epsilon}{\kappa (1-\beta)},  \quad \text{ implying }  X(n,n) = \omega + (r-1)(\delta-\omega) = 1, 
\]
is a feasible solution of \eqref{rechtLP} (note that $n = 2r$). By Lemma~\ref{lplem}, there exists $K$ sufficiently large such that any optimal solution $X^*$ must satisfy $X^*(n,n) = 1$. Using Lemma~\ref{lplem} again, there exists $K$ sufficiently large such that  $X^*(r,r) = 1$. 
%Since $\tr(X^*) = r$, the remaining weight to be assigned to the other diagonal entries of $X^*$ is at most $(r-2)+\mathcal{O}\left( \frac{1}{K} \right)$. 
Therefore, for $K$ sufficiently large, the $r$th and $n$th column of $\tilde{M}$ will be extracted implying 
\begin{align*} 
||W - \tilde{W}(:,P)||_{1} 
&  = \min_{1 \leq j \leq r-1} ||W(:,j) - M(:,n)||_1 \\
&  = ||W(:,1) - M(:,n)||_1 = \kappa \frac{r-2+\beta}{r-1} > \frac{r-2}{r-1} \kappa \geq \epsilon, 
\end{align*} 
and the proof will be complete. 

It remains to show that $X$ is feasible: Clearly, $\tr(X) = r$. For the constraints $0 \leq X \leq 1$, we check that   
\[
0 \leq \omega =  \frac{\epsilon}{\kappa(1-\beta)} = \frac{1}{r(1-\beta) + 1} \leq \delta  = (2-\beta) \omega =  \frac{(1-\beta) + 1}{r(1-\beta) + 1} \leq 1 ,  
\] 
and
 \[
 0 \leq \frac{1}{r-1} \left( \frac{1-\omega}{1-\delta} - \beta \right) \leq 1  \quad \text{ since } \;  \frac{1-\omega}{1-\delta} = \frac{r-1}{r-2} \geq 1. 
 \] 
%%\[
%%\omega + r(\delta - \omega) =  \left( 1 + r (1-\beta) \right) \frac{\epsilon}{\kappa(1-\beta)}  \leq 1 . 
%%\] 
For $X(i,j) \leq X(i,i)$ for all $i,j$, we only have to check that   
%\[
%\beta \geq \frac{1-\beta}{r-1} \iff \beta \geq \frac{1}{r}, 
%\]
%which is always satisfied since the columns of $H$ sum to one, and  
 \[
 1-\delta \geq \frac{\delta-\omega}{r-1} \iff (r-1) (r-2) (1-\beta) \geq (1-\beta). 
 \]
% is also satisfied since $\delta < 2 \omega$. 
 It remains to verify that  $||M(:,j)-MX(:,j)||_{1} \leq 2 \epsilon$ for all $1 \leq j \leq 2r$: 
\begin{itemize}

\item $1 \leq j \leq r-1$. Letting $\mathcal{J} = \{1,2,\dots,r\} \backslash \{j\}$, we have 
\begin{align*}
||M(:,j)-MX(:,j)||_{1} & = \left\|M(:,j)-(1-\delta) M(:,j) - \frac{\delta - \omega}{r-1} M(:,j+r) \right\|_{1} \\ 
& =  \left\|\delta M(:,j) - \omega M(:,j+r) - \frac{\delta - \omega}{r-1} M(:,\mathcal{J})e \right\|_{1} \\ 
& =  \omega \left\| M(:,j) -  M(:,j+r) \right\|_{1} \\ 
& \qquad \qquad + (\delta - \omega)  \left\|M(:,j)- \frac{1}{r-1} M(:,\mathcal{J})e \right\|_{1} \\ 
& =  \omega \kappa (1-\beta)  + (\delta - \omega) \kappa =   2 \omega \kappa (1-\beta) = 2 \epsilon . 
%& \leq   \frac{\kappa \delta}{2} \leq 2 \epsilon  \\ 
\end{align*}

\item $r+1 \leq j \leq 2r-1$. Letting $\mathcal{R} = \{1,2,\dots,r\} \backslash \{j-r\}$ and ${w}_j = W(:,\mathcal{R}) \frac{e}{r-1}$, we have %By construction, $M(:,j)-MX(:,j)$.
\begin{align*}
& ||M(:,j)-MX(:,j)||_{1}  \\ 
& \quad = \left\|M(:,j)- \omega M(:,j) - (1-\delta) \beta W(:,j-r) -  \left( 1 - \omega - \beta (1-\delta) \right) w_j \right\|_{1} \\ 
 %& = \left\|(1- \omega )M(:,j)- (1-\omega-(1-\beta) \omega) \beta M(:,j-r) -  \left( 1 - \omega - \beta (1-\delta) \right)  {w}_j \right\|_{1} \\
 & \quad= (1- \omega ) \left\|M(:,j)- \frac{r-2}{r-1} \beta M(:,j-r) 
-  \left(1-\beta \frac{r-2}{r-1}\right)  {w}_j \right\|_{1} \\
 &\quad = (1- \omega ) \left\|M(:,j)- \left(1-\frac{1}{r-1}\right) \beta W(:,j-r) 
-  \left(1-\beta + \beta \frac{1}{r-1}\right)  {w}_j \right\|_{1} \\
 & \quad= \frac{\beta(1- \omega )}{r-1} \left\|  W(:,j-r) -  {w}_j \right\|_{1} = \frac{\beta(1- \omega ) \kappa}{r-1} \leq \frac{(1-\omega) \kappa}{r} \\
 & \quad=  \frac{(r-1) (1-\beta) \kappa}{r((r-1)(1-\beta)+1) } \leq \frac{(1-\beta) \kappa}{(r-1)(1-\beta)+1 } = \epsilon, 
\end{align*}
In fact, $\frac{1-\delta}{1-\omega} = \frac{r-2}{r-1}$, $\beta \leq \frac{1}{r}$, and, by construction, $M(:,j) = \beta W(:,j-r) + (1-\beta) w_j$.
\end{itemize} 
  \end{proof}

\section{Proof of Theorem~\ref{th7}} \label{appc}

\begin{proof}[\textbf{Proof of Theorem~\ref{th7}}]
We prove the result with the following construction: Let 
\[
W = \left(
\begin{array}{c}
\frac{\kappa}{2} I_{r}  \\
(1-\frac{\kappa}{2})  e^T  \\ 
0_{r \times r}  
\end{array} 
\right) ,  \quad 
%\] 
%\[
H = \left( 
\begin{array}{cccc}
I_{r-1} & 0 &  \lambda I_{r-1} &  \frac{1}{r-1} e \\
 0 & 1 &  (1-\lambda) e^T &  0  \\
\end{array} 
\right) ,  
\] 
where $\lambda = 2 \frac{\epsilon}{\kappa}$,  
\[
N = \left( \begin{array}{cccc }
0_{(r+1) \times r} &  0_{(r+1) \times 1} &  0_{(r+1) \times (r-1)} & 0_{(r+1) \times 1} \\
\epsilon e^T &  0 &  0_{1 \times (r-1)} & \epsilon \\
%0_{1 \times r} &  \epsilon &  0_{1 \times (r-1)} & 0 \\
0_{(r-1) \times (r-1)} &  0_{(r-1) \times 1} &  Z & 0_{(r-1) \times 1} \\ 
\end{array} \right), 
\]
where $Z = x I_{r-1} + y (E_{r-1} - I_{r-1})$ with $x = \frac{1}{r-1} \epsilon$ and $y = \frac{-x}{r-2}$. The matrix $Z$ has been constructed so that $||Z(:,j)||_1 \leq \epsilon$ for all $j$, $\sum_j Z(i,j) = 0$ for all $i$, and $\left\|\tilde{M}(:,j) - \frac{1}{r-1} \tilde{M}(:,\mathcal{I}) e\right\|_1 = 2\epsilon$ for all $j \in \mathcal{J} = \{r+1,r+2,\dots 2r-1\}$ and $\mathcal{I} = \mathcal{J} \backslash \{i\}$.  
Let also $\tilde{M} = WH + N$,   %Note that the only difference with Lemma~\ref{lem7} is the last column of $\tilde{M}$ which belong to the convex hull of the first (r-1). Hence Lemma~\ref{lem7} also applies to this particular case. Let also 
%(Note that $W$ is $\alpha$-robustly simplicial, and $||N||_{1} = \epsilon$.) 
\[
\frac{\kappa}{(r-1)^2} < \epsilon \leq \frac{\kappa}{2(r-1)} \quad \text{ so that } \lambda \leq \frac{1}{r-1}, 
\]
and 
\[
p = ( 1, 2, \dots, r-1, K^3, K^2,K^2+1,\dots,K^2+r-1, -K )^T, 
\]
for $K$ sufficiently large. %Let us also denote $J = \{ r+1,r+1,\dots,2r-1\}$. 
Assume
\[
X = \left( \begin{array}{cccc}
(1-\frac{2\epsilon}{\kappa}) I_{r-1}  & 0  & \lambda \frac{r-2}{r-1}  I_{r-1} &  \left(1-\frac{2\epsilon (r-1)}{ \kappa}\right)  e^T  \\
  0 & 0  & 0  & 0 \\
 0  & \frac{1}{r-1} e  &  \frac{1}{r-1} E_{r-1}  & 0 \\
 \frac{2\epsilon}{\kappa} e^T & 0  & 0 & (r-1) \frac{2\epsilon}{\kappa}  \\
\end{array} \right) , 
\] 
 is feasible for \eqref{rechtLP}. Letting $X^*$ be any optimal solution, by Lemma~\ref{lplem}, there exists $K$ sufficiently large such that $X^*(r,r) = 0$. By Lemma~\ref{optsol} (see below), this implies that $X^*(j,j) \geq \frac{1}{r-1}$ for $j \in \mathcal{J}$. Using Lemma~\ref{lplem} again, we  have that for $K$ sufficiently large $X^*(j,j) = \frac{1}{r-1}$ for all for $j \in \mathcal{J}$, and $X^*(n,n) \geq (r-1) \frac{2\epsilon}{\kappa}$. 
Therefore, since 
\[
\frac{2\epsilon (r-1)}{\kappa} > \frac{1}{r-1} \iff {\epsilon} > \frac{\kappa}{2(r-1)^2}, 
\]
and 
\[
1-\frac{2\epsilon}{\kappa} > \frac{1}{r-1} \iff \epsilon < \left(\frac{r-2}{r-1}\right) \frac{\kappa}{2} \leq \frac{\kappa}{4} , 
\]
 the first $r-1$ columns and the last column of $\tilde{M}$ will be extracted so that
\[
||W - \tilde{W} ||_{1} = ||W(:,r) - \tilde{M}(:,n)||_1 = \kappa + \epsilon, 
\]
and the proof will be complete. 

It remains to show that $X$ is feasible. We clearly have $\tr(X) = r$, $0 \leq X \leq 1$, and $X(i,j) \leq X(i,i)$ for all $i,j$ because $\epsilon \leq \frac{\kappa}{2(r-1)}$ while, for $||\tilde{M}(:,j) - \tilde{M}X(:,j)||_{1} \leq 2\epsilon$ for all $j$, we have 
\begin{itemize}

\item $1 \leq j \leq r-1$. 
\[
||\tilde{M}(:,j) - \tilde{M}X(:,j)||_{1} = \frac{2\epsilon}{\kappa} ||\tilde{M}(:,j) - \tilde{M}(:,n)||_1 = 2 \frac{r-2}{r-1}\epsilon \leq 2\epsilon. 
\]
\item $j = r$. This follows from Lemma~\ref{optsol}.
 
\item $r+1 \leq j \leq 2r-1$. This follows from the construction of matrix $Z$. 

\item $j = 2r$. $\tilde{M}(:,j) = \tilde{M}X(:,j)$ since $\tilde{M}(:,j) = \frac{1}{r-1}W(:,1\text{:}r-1)e$.
 
\end{itemize}
\end{proof}

\begin{lemma} \label{optsol} 
Let $W, H, N$ and $\tilde{M} = WH+N$ be the matrices constructed in Theorem~\ref{th7}. Let also $\mathcal{R} = \{1,2,\dots,n\} \backslash \{r\}$. 
Then 
\begin{equation} \label{optx}
\min_{x \geq 0} ||\tilde{M}(:,r) - \tilde{M}(:,\mathcal{R})x||_1 = 2 \epsilon, 
\end{equation}
and the \emph{unique} optimal solution of \eqref{optx} is given by 
\[
x^{\dagger} = 
\left( \begin{array}{c}
0_{(r-1) \times 1}  \\
\frac{1}{r-1} e \\
0 _{1 \times 1}
\end{array} \right) \; \in \mathbb{R}^{2r-1}.  
\]
\end{lemma}
\begin{proof} 
Let $x^{*} = (y , z , w)$ be an optimal solution of \eqref{optx} where $y, z \in \mathbb{R}_+^{r-1}$ and $w \in \mathbb{R}_+$. We have to show that $x^* = x^{\dagger}$. 
%Then, let us show that $y=0$. Without loss of generality, we can assume the all the entries of $z$ are equal to each other. In fact, the problem is perfectly symmetric in the entries of $z$, that is, if if $z$ is optimal then --- We must have $z \neq 0$: in fact, if $z = 0$, then $f(x) = \frac{\alpha}{2} > 2\epsilon$.
From $x^*$, let us construct another optimal solution $x' = (y',z',0)$ such that all the entries of $y'$ and $z'$ are equal to each other. 
Because $\tilde{M}(:,n) = \frac{1}{r-1}\tilde{M}(:,1\text{:}r-1)e$, we take $w = 0$, replace $z \leftarrow z + \frac{w}{r-1}$ and obtain an equivalent solution. Let us denote $\bar{\mathcal{R}} = \mathcal{R} \backslash \{n\}$ and 
\[
g(y,z) = \left\|\tilde{M}(:,r) - \tilde{M}(:,\bar{\mathcal{R}})\binom{y}{z}\right\|_1. 
\] 
By symmetry, one can check that $g(y,z) =  g(y(P),z(P))$ for any permutation $P$ of $\{1,2,\dots,r-1\}$ (this simply amounts to permuting the first and last $r-1$ columns of $M(:,\bar{\mathcal{R}})$). By convexity, $(y',z') = \frac{1}{|\Pi|}\sum_{P \in \Pi}  (y(P),z(P))$, where $\Pi$ is the set of all possible permutations of  $\{1,2,\dots,r-1\}$, is also an optimal solution of \eqref{optx} hence all entries of $y'$ and $z'$ are equal to each other, and $||y'||_1 =  ||y||_1$ and $||z'||_1 =  ||z||_1$. 
 
Therefore, denoting $y'(i)= \frac{a}{r-1}$ and $z'(i)= \frac{b}{r-1}$ for all $1 \leq i \leq r-1$, the optimization problem \eqref{optx} can be reduced to 
\begin{equation} \label{eq1}
\min_{a, b \geq 0}  \left\|
\left( \begin{array}{c}
  0_{(r-1) \times 1}  \\
 \frac{\kappa}{2}  \\ 
 1-\frac{\kappa}{2} \\
 0 \\ 
 0_{(r-1) \times 1} 
\end{array} \right) 
- a \left( \begin{array}{c}
 \frac{\kappa}{2(r-1)} e  \\
 0 \\ 
  1-\frac{\kappa}{2} \\
 \epsilon \\
 0_{(r-1) \times 1} \\
\end{array} \right) 
- b \left( \begin{array}{c}
   \frac{\lambda \kappa}{2(r-1)}  e  \\
 (1-\lambda) \frac{ \kappa}{2} \\ 
  1-\frac{\kappa}{2} \\
  0 \\ 
 0_{(r-1) \times 1} 
\end{array} \right) \right\|_1
\end{equation}
\[
\equiv 
\min_{a, b \geq 0}  \; h(a,b) = 
\frac{\kappa}{2} \left| a + \lambda b \right| 
+ 
\frac{\kappa}{2} \left| 1- (1-\lambda) b \right|  
+
\left( 1 - \frac{\kappa}{2} \right) \left| 1- a -  b \right|  
+
\epsilon \left| a \right|.   
\]
Let us show that $(a^*,b^*) = (0,1)$ is the unique optimal solution, for which $h(0,1) = {\kappa \lambda} = 2\epsilon$. 
First, note that $(0,0)$ cannot be optimal since $h(0,0) = 1$. 
%For the subdifferential in $a$ to be equal to zero, it is required that $a^*+b^* \leq 1$. 
For $a+b > 1$, the subdifferential of $h$ in $a$ is larger than $1 - \epsilon > 0$, while, 
for $0 < a+b < 1$, the subdifferential of $h$ in $b$ is $\kappa\lambda - 1 = 2\epsilon - 1 < 0$ (recall that $\lambda = \frac{2\epsilon}{\kappa}$, $\epsilon \leq \frac{\kappa}{2(r-1)}$, $\kappa \leq 1$, and $r \geq 3$)  hence $a^*+b^* = 1$ at optimality. Substituting $a = 1 - b$ above, we obtain 
\[
b^* = \argmin_{0 \leq b \leq 1} \; 2 |1- (1-\lambda)b| = 1, 
\]
which is unique as the slope at $b=1$ is negative (since $0 \leq \lambda \leq 1/2$). 

Finally, we have $b^* = 1$, $a^* = 0$ is the unique solution of \eqref{eq1} implying that $y' = y = 0$ and that the minimal objective function value of \eqref{optx} is 2$\epsilon$. Moreover, this implies $||z'||_1 = ||z||_1 = 1$. 
It remains to show that the entries of $z$ are equal to each other, that is, show that the unique solution to the following system 
\[
  \left\|
\left( \begin{array}{c}
  0_{(r-1) \times 1}  \\
 \frac{\kappa}{2}  \\ 
 0_{(r-1) \times 1} 
\end{array} \right) 
-  \left( \begin{array}{c}
  \frac{\lambda \kappa}{2}  I_{r-1}   \\
  (1-\lambda) \frac{\kappa}{2} \\
 Z 
\end{array} \right) z \right\|_1 = 2 \epsilon, 
\] 
is $z^* = \frac{1}{r-1} e$, which is clearly the case as the only $z$ such that $Zz = 0$ and $||z||_1 = 1$ is $z^*$. This completes the proof.  
\end{proof}

\bibliographystyle{siam} 
\bibliography{Biography}

\end{document}